\documentclass[10pt,twocolumn,letterpaper]{article}
\usepackage{iccv}
\usepackage{epsfig}
\usepackage{graphicx}
\usepackage{times,helvet,courier}
\usepackage[tight,footnotesize]{subfigure}
\usepackage{amsmath,amssymb,amsthm,bm}
\usepackage{multirow,makecell,footmisc,url}
\usepackage[ruled]{algorithm2e}
\usepackage{color}

\newtheorem{lemma}{Lemma}
\newtheorem{theorem}{Theorem}
\newtheorem{remark}{Case}
\usepackage[pagebackref=true,breaklinks=true,letterpaper=true,colorlinks,bookmarks=false]{hyperref}

\iccvfinalcopy
\def\iccvPaperID{3025}

\ificcvfinal\fi

\begin{document}

\title{HashNet: Deep Learning to Hash by Continuation\thanks{Corresponding author: M. Long (mingsheng@tsinghua.edu.cn).}}

\author{
    Zhangjie Cao$^\dag$, Mingsheng Long$^\dag$, Jianmin Wang$^\dag$, and Philip S. Yu$^{\dag\ddag}$\\
    $^\dag$KLiss, MOE; NEL-BDS; TNList; School of Software, Tsinghua University, China\\
    $^\ddag$University of Illinois at Chicago, IL, USA\\
    {\tt\small caozhangjie14@gmail.com\quad \{mingsheng,jimwang\}@tsinghua.edu.cn\quad psyu@uic.edu}
}

\maketitle

\begin{abstract}
Learning to hash has been widely applied to approximate nearest neighbor search for large-scale multimedia retrieval, due to its computation efficiency and retrieval quality. Deep learning to hash, which improves retrieval quality by end-to-end representation learning and hash encoding, has received increasing attention recently. Subject to the ill-posed gradient difficulty in the optimization with sign activations, existing deep learning to hash methods need to first learn continuous representations and then generate binary hash codes in a separated binarization step, which suffer from substantial loss of retrieval quality. This work presents HashNet, a novel deep architecture for deep learning to hash by continuation method with convergence guarantees, which learns exactly binary hash codes from imbalanced similarity data. The key idea is to attack the ill-posed gradient problem in optimizing deep networks with non-smooth binary activations by continuation method, in which we begin from learning an easier network with smoothed activation function and let it evolve during the training, until it eventually goes back to being the original, difficult to optimize, deep network with the sign activation function. Comprehensive empirical evidence shows that HashNet can generate exactly binary hash codes and yield state-of-the-art multimedia retrieval performance on standard benchmarks.
\end{abstract}

\vspace{-20pt}

\section{Introduction}
In the big data era, large-scale and high-dimensional media data has been pervasive in search engines and social networks. To guarantee retrieval quality and computation efficiency, approximate nearest neighbors (ANN) search has attracted increasing attention. Parallel to the traditional indexing methods \cite{cite:TOMM06CBIR}, another advantageous solution is hashing methods \cite{cite:Arxiv14HashSurvey}, which transform high-dimensional media data into compact binary codes and generate similar binary codes for similar data items. In this paper, we will focus on learning to hash methods \cite{cite:Arxiv14HashSurvey} that build data-dependent hash encoding schemes for efficient image retrieval, which have shown better performance than data-independent hashing methods, e.g. Locality-Sensitive Hashing (LSH) \cite{cite:VLDB99LSH}.

Many learning to hash methods have been proposed to enable efficient ANN search by Hamming ranking of compact binary hash codes \cite{cite:NIPS09BRE,cite:CVPR11ITQ,cite:ICML11MLH,cite:CVPR12MIH,cite:CVPR12KSH,cite:TPAMI12SSH,cite:CVPR13HBS,cite:CVPR13BP,cite:ICML14CBE,cite:SIGIR14LFH}.
Recently, deep learning to hash methods \cite{cite:AAAI14CNNH,cite:CVPR15DNNH,cite:CVPR15SDH,cite:CVPR15DH,cite:AAAI16DHN,cite:IJCAI16DPSH,cite:CVPR2016DSH} have shown that end-to-end learning of feature representation and hash coding can be more effective using deep neural networks \cite{cite:NIPS12CNN,cite:TPAMI13DL}, which can naturally encode any nonlinear hash functions. These deep learning to hash methods have shown state-of-the-art performance on many benchmarks. In particular, it proves crucial to jointly learn similarity-preserving representations and control quantization error of binarizing continuous representations to binary codes \cite{cite:AAAI16DHN,cite:IJCAI16DPSH,cite:CVPR15DSRH,cite:CVPR2016DSH}. However, a key disadvantage of these deep learning to hash methods is that they need to first learn continuous deep representations, which are binarized into hash codes in a separated post-step of sign thresholding. By \emph{continuous relaxation}, i.e.  solving the discrete optimization of hash codes with continuous optimization, all these methods essentially solve an optimization problem that deviates significantly from the hashing objective as they cannot learn \emph{exactly binary} hash codes in their optimization procedure. Hence, existing deep hashing methods may fail to generate compact binary hash codes for efficient similarity retrieval.

There are two key challenges to enabling deep learning to hash truly end-to-end. First, converting deep representations, which are \emph{continuous} in nature, to \emph{exactly binary} hash codes, we need to adopt the \emph{sign} function $h = \operatorname{sgn} \left( z \right)$ as activation function when generating binary hash codes using similarity-preserving learning in deep neural networks. However, the gradient of the sign function is zero for all nonzero inputs, which will make standard back-propagation infeasible. This is known as the \emph{ill-posed gradient} problem, which is the key difficulty in training deep neural networks via back-propagation \cite{cite:NC06DBN}. Second, the similarity information is usually very sparse in real retrieval systems, i.e., the number of similar pairs is much smaller than the number of dissimilar pairs. This will result in the \emph{data imbalance} problem, making similarity-preserving learning ineffective. Optimizing deep networks with \emph{sign} activation remains an open problem and a key challenge for deep learning to hash.

This work presents \textbf{HashNet}, a new architecture for deep learning to hash by continuation  with convergence guarantees, which addresses the ill-posed gradient and data imbalance problems in an end-to-end framework of deep feature learning and binary hash encoding. 
Specifically, we attack the \emph{ill-posed gradient} problem in the non-convex optimization of the deep networks with non-smooth sign activation by the \emph{continuation} methods \cite{cite:Book12Continuation}, which address a complex optimization problem by smoothing the original function, turning it into a different problem that is easier to optimize. By gradually reducing the amount of smoothing during the training, it results in a sequence of optimization problems converging to the original optimization problem. A novel weighted pairwise cross-entropy loss function is designed for similarity-preserving learning from imbalanced similarity relationships. Comprehensive experiments testify that HashNet can generate exactly binary hash codes and yield state-of-the-art retrieval performance on standard datasets.

\section{Related Work}
Existing learning to hash methods can be organized into two categories: unsupervised hashing and supervised hashing. We refer readers to \cite{cite:Arxiv14HashSurvey} for a comprehensive survey.

Unsupervised hashing methods learn hash functions that encode data points to binary codes by training from unlabeled data. Typical learning criteria include reconstruction error minimization \cite{cite:AI07SemanticHashing,cite:CVPR11ITQ,cite:TPAMI11PQ} and graph learning\cite{cite:NIPS09SH,cite:ICML11AGH}. While unsupervised methods are more general and can be trained without semantic labels or relevance information, they are subject to the semantic gap dilemma \cite{cite:TPAMI00SemanticGap} that high-level semantic description of an object differs from low-level feature descriptors. Supervised methods can incorporate semantic labels or relevance information to mitigate the semantic gap and improve the hashing quality significantly. 
Typical supervised methods include Binary Reconstruction Embedding (BRE) \cite{cite:NIPS09BRE}, Minimal Loss Hashing (MLH) \cite{cite:ICML11MLH} and Hamming Distance Metric Learning \cite{cite:NIPS12HDML}.
Supervised Hashing with Kernels (KSH) \cite{cite:CVPR12KSH} generates hash codes by minimizing the Hamming distances across similar pairs and maximizing the Hamming distances across dissimilar pairs.

As deep convolutional neural network (CNN) \cite{cite:NIPS12CNN,cite:CVPR16DRL} yield breakthrough performance on many computer vision tasks, deep learning to hash has attracted attention recently. CNNH \cite{cite:AAAI14CNNH} adopts a two-stage strategy in which the first stage learns hash codes and the second stage learns a deep network to map input images to the hash codes. DNNH \cite{cite:CVPR15DNNH} improved the two-stage CNNH with a simultaneous feature learning and hash coding  pipeline such that representations and hash codes can be optimized in a joint learning process. DHN \cite{cite:AAAI16DHN} further improves DNNH by a cross-entropy loss and a quantization loss which preserve the pairwise similarity and control the quantization error simultaneously. DHN obtains state-of-the-art performance on several benchmarks.

However, existing deep learning to hash methods only learn continuous codes ${\bm g}$ and need a binarization post-step to generate binary codes ${\bm h}$. By continuous relaxation, these methods essentially solve an optimization problem $ L({\bm g})$ that deviates significantly from the hashing objective $ L({\bm h})$, because they cannot keep the codes exactly binary after convergence. Denote by $Q({\bm g}, {\bm h})$ the quantization error function by binarizing continuous codes ${\bm g}$ into binary codes ${\bm h}$. Prior methods control the quantization error in two ways: \textbf{(a)} $\min L({\bm g}) + Q({\bm g}, {\bm h})$ through continuous optimization \cite{cite:AAAI16DHN,cite:IJCAI16DPSH}; \textbf{(b)} $\min L({\bm h}) + Q({\bm g}, {\bm h})$ through discrete optimization on $L({\bm h})$ but continuous optimization on $Q({\bm g}, {\bm h})$ (the continuous optimization is used for out-of-sample extension as discrete optimization cannot be extended to the test data) \cite{cite:CVPR2016DSH}. However, since $Q({\bm g}, {\bm h})$ cannot be minimized to zero, there is a large gap between continuous codes and binary codes. To directly optimize $\min L({\bm h})$, we must adopt \emph{sign} as the \emph{activation} function \emph{within} deep networks, which enables generation of exactly binary codes but introduces the \emph{ill-posed gradient} problem. This work is the first effort to learn sign-activated deep networks by continuation method, which can directly optimize $L({\bm h})$ for deep learning to hash.

\begin{figure*}[tbp]
  \centering
  \includegraphics[width=0.66\textwidth]{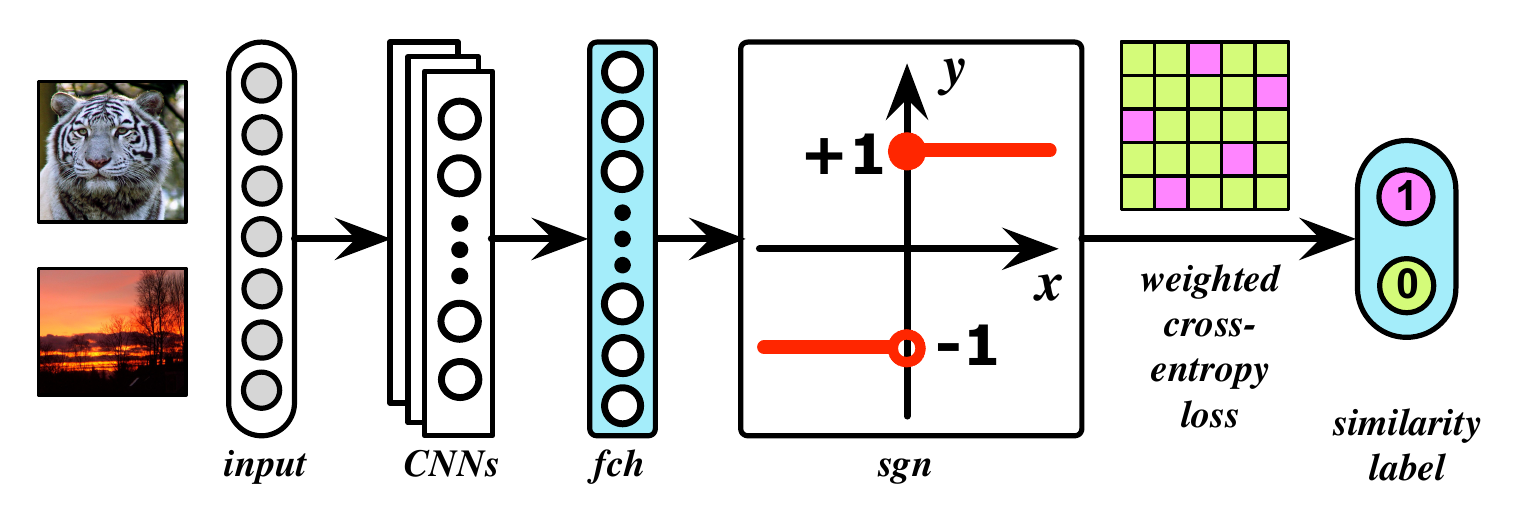}
  \includegraphics[width=0.33\textwidth]{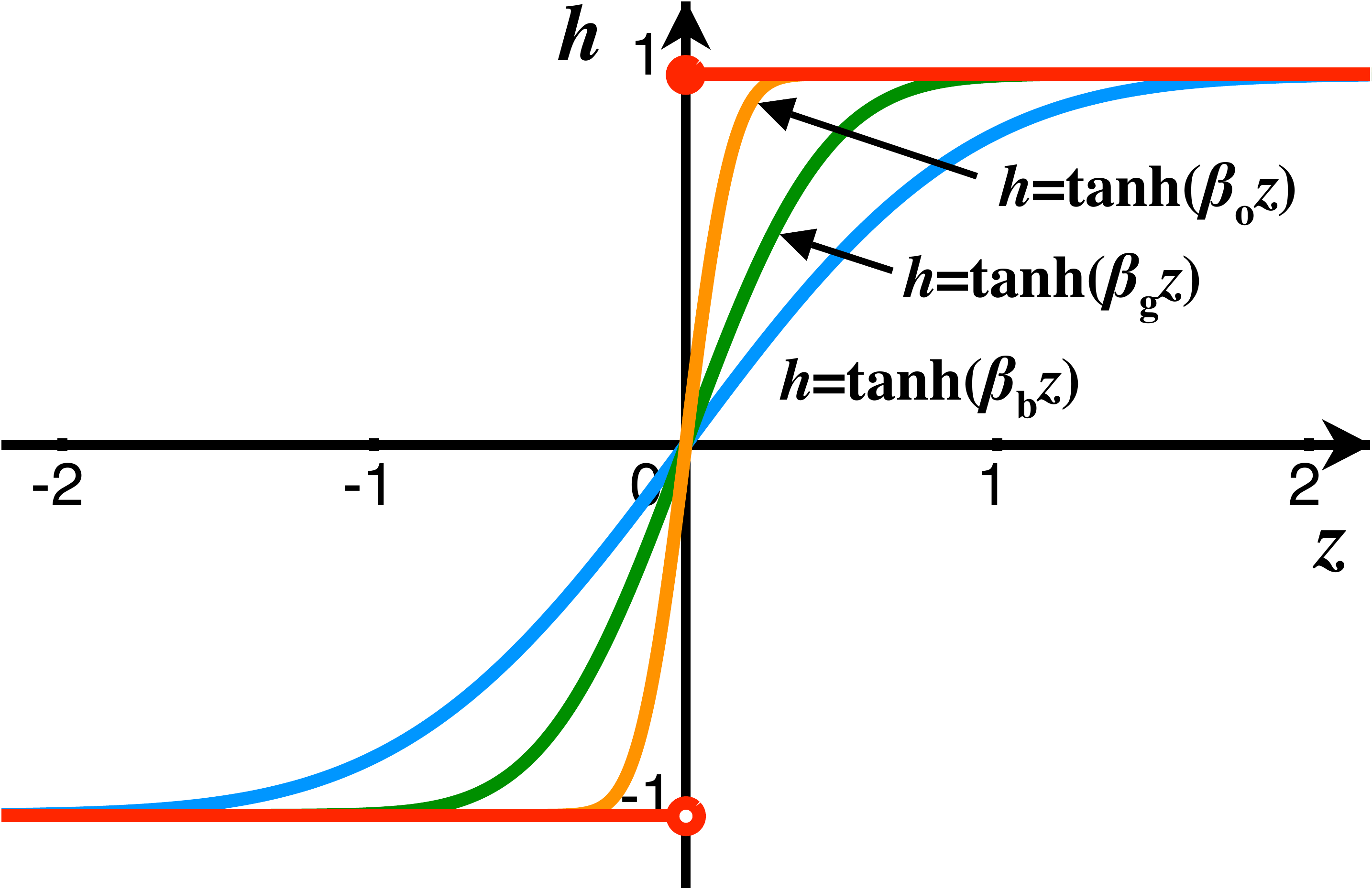}
  \caption{(left) The proposed HashNet for deep learning to hash by continuation, which is comprised of four key components: (1) Standard convolutional neural network (CNN), e.g. AlexNet and ResNet, for learning deep image representations, (2) a fully-connected hash layer ($fch$) for transforming the deep representation into $K$-dimensional representation, (3) a sign activation function ($\operatorname{sgn}$) for binarizing the $K$-dimensional representation into $K$-bit binary hash code, and (4) a novel weighted cross-entropy loss for similarity-preserving learning from sparse data. (right) Plot of smoothed responses of the sign function $h = \operatorname{sgn} \left( z \right)$: Red is the sign function, and blue, green and orange show functions $h = \operatorname{tanh} \left( \beta z \right)$ with bandwidths $\beta_{\text{b}} < \beta_{\text{g}} < \beta_{\text{o}}$. The key property is $\mathop {\lim }\nolimits_{\beta  \to \infty } \tanh \left( {\beta z} \right) = \operatorname{sgn} \left( z \right)$. \textit{Best viewed in color}.}
  \label{fig:HashNet}
  \vspace{-10pt}
\end{figure*}

\section{HashNet}
In similarity retrieval systems, we are given a training set of $N$ points $\{{\bm x}_i\}_{i=1}^N$, each represented by a $D$-dimensional feature vector ${\bm x}_i \in \mathbb{R}^{D}$.
Some pairs of points ${\bm x}_i$ and ${\bm x}_j$ are provided with similarity labels $s_{ij}$, where $s_{ij} = 1$ if ${\bm x}_i$ and ${\bm x}_j$ are similar while $s_{ij} = 0$ if ${\bm x}_i$ and ${\bm x}_j$ are dissimilar. 
The goal of deep learning to hash is to learn nonlinear hash function $f:{\bm{x}} \mapsto {\bm{h}} \in {\left\{ { - 1,1} \right\}^K}$ from input space $\mathbb{R}^D$ to Hamming space $\{-1,1\}^K$ using deep neural networks, which encodes each point ${\bm x}$ into compact $K$-bit binary hash code ${\bm h} = f({\bm x})$ such that the similarity information between the given pairs ${\cal S}$ can be preserved in the compact hash codes. In supervised hashing, the similarity set $\mathcal{S} = \{s_{ij}\}$ can be constructed from semantic labels of data points or relevance feedback from click-through data in real retrieval systems.

To address the data imbalance and ill-posed gradient problems in an end-to-end learning framework, this paper presents \textbf{HashNet}, a novel architecture for deep learning to hash by continuation, shown in Figure~\ref{fig:HashNet}. 
The architecture accepts pairwise input images $\{({\bm x}_i, {\bm x}_j, s_{ij})\}$ and processes them through an end-to-end pipeline of deep representation learning and binary hash coding: (1) a convolutional network (CNN) for learning deep representation of each image ${\bm x}_i$, (2) a fully-connected hash layer ($fch$) for transforming the deep representation into $K$-dimensional representation ${\bm z}_i \in \mathbb{R}^K$, (3) a sign activation function $h = \operatorname{sgn} \left( z \right)$ for binarizing the $K$-dimensional representation ${\bm z}_i$ into $K$-bit binary hash code ${\bm h}_i \in \{-1,1\}^K$, and (4) a novel weighted cross-entropy loss for similarity-preserving learning from imbalanced data. 
We attack the ill-posed gradient problem of the non-smooth activation function $h = \operatorname{sgn} \left( z \right)$ by continuation, which starts with a smoothed activation function $y = \operatorname{tanh} \left( \beta x \right)$ and becomes more non-smooth by increasing $\beta$ as the training proceeds, until eventually goes back to the original, difficult to optimize, sign activation function.

\subsection{Model Formulation}
To perform deep learning to hash from imbalanced data, we jointly preserve similarity information of pairwise images and generate binary hash codes by weighted maximum likelihood \cite{cite:JMLR10MWL}. For a pair of binary hash codes ${\bm h}_i$ and ${\bm h}_j$, there exists a nice relationship between their Hamming distance $\mathrm{dist}_H(\cdot,\cdot)$ and inner product $\langle \cdot,\cdot \rangle$: ${\textrm{dis}}{{\text{t}}_H}\left( {{{\bm{h}}_i},{{\bm{h}}_j}} \right) = \frac{1}{2}\left( {K - \left\langle {{{\bm{h}}_i},{{\bm{h}}_j}} \right\rangle } \right)$. Hence, the Hamming distance and inner product can be used interchangeably for binary hash codes, and we adopt inner product to quantify pairwise similarity. Given the set of pairwise similarity labels $\mathcal{S} = \{s_{ij}\}$, the Weighted Maximum Likelihood (WML) estimation of the hash codes ${\bm H} = [{\bm h}_1,\ldots,{\bm h}_N]$ for all $N$ training points is
\begin{equation}\label{eqn:WML}
	\log P\left( {\mathcal{S}|{\bm{H}}} \right) = \sum\limits_{{s_{ij}} \in \mathcal{S}} {{w_{ij}}\log P\left( {{s_{ij}}|{{\bm{h}}_i},{{\bm{h}}_j}} \right)} ,
\end{equation}
where $P({\cal S}|{\bm H})$ is the weighted likelihood function, and $w_{ij}$ is the weight for each training pair $({\bm x}_i, {\bm x}_j, s_{ij})$, which is used to tackle the data imbalance problem by weighting the training pairs according to the importance of misclassifying that pair \cite{cite:JMLR10MWL}. Since each similarity label in ${\cal S}$ can only be $s_{ij} = 1$ (similar) or $s_{ij} = 0$ (dissimilar), to account for the data imbalance between similar and dissimilar pairs, we set 
\begin{equation}\label{eqn:Weight}
{w_{ij}} = c_{ij} \cdot
  \begin{cases}
    \left| {{\mathcal{S}}} \right|/\left| {{\mathcal{S}_1}} \right|,\quad {s_{ij}} = 1 \hfill \\
    \left| {{\mathcal{S}}} \right|/\left| {{\mathcal{S}_0}} \right|,\quad {s_{ij}} = 0 \hfill \\ 
  \end{cases}
\end{equation}
where ${\mathcal{S}_1} = \left\{ {{s_{ij}} \in \mathcal{S}:{s_{ij}} = 1} \right\}$ is the set of similar pairs and ${\mathcal{S}_0} = \left\{ {{s_{ij}} \in \mathcal{S}:{s_{ij}} = 0} \right\}$ is the set of dissimilar pairs; $c_{ij}$ is continuous similarity, i.e. $c_{ij} = \frac{{{\bm{y}_i} \cap {\bm{y}_j}}}{{{\bm{y}_i} \cup {\bm{y}_j}}}$ if labels $\bm{y}_i$ and $\bm{y}_j$ of $\bm{x}_i$ and $\bm{x}_j$ are given, $c_{ij}=1$ if only $s_{ij}$ is given.
For each pair, $P(s_{ij}|{\bm h}_i,{\bm h}_j)$ is the conditional probability of similarity label $s_{ij}$ given a pair of hash codes ${\bm h}_i$ and ${\bm h}_j$, which can be naturally defined as pairwise logistic function,
\begin{equation}\label{eqn:PLR}
	\begin{aligned}
	P\left( {{s_{ij}}|{{\bm{h}}_i},{{\bm{h}}_j}} \right) &= 
		\begin{cases}
			\sigma \left( {\left\langle {{\bm{h}}_i, {\bm{h}}_j} \right\rangle } \right), & {s_{ij}} = 1 \\
			1 - \sigma \left( {\left\langle {{\bm{h}}_i, {\bm{h}}_j} \right\rangle } \right), & {s_{ij}} = 0 \\
		\end{cases} \\
		& = \sigma {\left( {\left\langle {{\bm{h}}_i,{\bm{h}}_j} \right\rangle } \right)^{{s_{ij}}}}{\left( {1 - \sigma \left( {\left\langle {{\bm{h}}_i,{\bm{h}}_j} \right\rangle } \right)} \right)^{1 - {s_{ij}}}} \\
	\end{aligned}
\end{equation}
where $\sigma \left( x \right) = {1}/({{1 + {e^{ - \alpha x}}}})$ is the \emph{adaptive} sigmoid function with hyper-parameter $\alpha$ to control its bandwidth. Note that the sigmoid function with larger $\alpha$ will have larger saturation zone where its gradient is zero. To perform more effective back-propagation, we usually require $\alpha < 1$, which is more effective than the typical setting of $\alpha = 1$. Similar to logistic regression, we can see in pairwise logistic regression that the smaller the Hamming distance ${\text{dis}}{{\textrm{t}}_H}\left( {{{\bm{h}}_i},{{\bm{h}}_j}} \right)$ is, the larger the inner product ${\left\langle {{{\bm{h}}_i},{{\bm{h}}_j}} \right\rangle }$ as well as the conditional probability $P\left( {1|{{\bm{h}}_i},{{\bm{h}}_j}} \right)$ will be, implying that pair ${\bm h}_i$ and ${\bm h}_j$ should be classified as similar; otherwise, the larger the conditional probability $P\left( {0|{{\bm{h}}_i},{{\bm{h}}_j}} \right)$ will be, implying that pair ${\bm h}_i$ and ${\bm h}_j$ should be classified as dissimilar. Hence, Equation~\eqref{eqn:PLR} is a reasonable extension of the logistic regression classifier to the pairwise classification scenario, which is optimal for binary similarity labels $s_{ij}\in\{0,1\}$.

By taking Equation \eqref{eqn:PLR} into WML estimation in Equation~\eqref{eqn:WML}, we achieve the optimization problem of HashNet,
\begin{equation}\label{eqn:model}
	\mathop {\min }\limits_{\Theta}  \sum\limits_{{s_{ij}} \in \mathcal{S}} {{w_{ij}}\left( {\log \left( {1 + \exp \left( {\alpha \left\langle {{{\bm{h}}_i},{{\bm{h}}_j}} \right\rangle } \right)} \right) - \alpha {s_{ij}}\left\langle {{{\bm{h}}_i},{{\bm{h}}_j}} \right\rangle } \right)},
\end{equation}
where ${\Theta}$ denotes the set of all parameters in deep networks. Note that, HashNet directly uses the sign activation function ${{\bm{h}}_i} = \operatorname{sgn} \left( {{{\bm{z}}_i}} \right)$ which converts the $K$-dimensional representation to \emph{exactly} binary hash codes, as shown in Figure~\ref{fig:HashNet}. By optimizing the WML estimation in Equation~\eqref{eqn:model}, we can enable deep learning to hash from imbalanced data under a statistically optimal framework. It is noteworthy that our work is the first attempt that extends the WML estimation from pointwise scenario to pairwise scenario. The HashNet can jointly preserve similarity information of pairwise images and generate \emph{exactly} binary hash codes. Different from HashNet, previous deep-hashing methods need to first learn continuous embeddings, which are binarized in a separated step using the sign function. This will result in substantial quantization errors and significant losses of retrieval quality.

\subsection{Learning by Continuation}
HashNet learns \emph{exactly} binary hash codes by converting the $K$-dimensional representation ${\bm z}$ of the hash layer $fch$, which is continuous in nature, to binary hash code ${\bm h}$ taking values of either $+1$ or $-1$. This binarization process can only be performed by taking the sign function $h = \operatorname{sgn} \left( z \right)$ as activation function on top of hash layer $fch$ in HashNet,
\begin{equation}\label{eqn:sgn}
  h = \operatorname{sgn} \left( z \right) =
  \begin{cases}
   + 1,\quad {\text{if}}\;z \geqslant 0 \\
   - 1,\quad {\text{otherwise}} \\ 
  \end{cases}
\end{equation}
Unfortunately, as the sign function is non-smooth and non-convex, its gradient is zero for all nonzero inputs, and is ill-defined at zero, which makes the standard back-propagation infeasible for training deep networks. This is known as the \emph{vanishing gradient} problem, which has been a key difficulty in training deep neural networks via back-propagation \cite{cite:NC06DBN}. 

Many optimization methods have been proposed to circumvent the vanishing gradient problem and enable effective network training with back-propagation, including unsupervised pre-training \cite{cite:NC06DBN,cite:NIPS07LWT}, dropout \cite{cite:JMLR14Dropout}, batch normalization \cite{cite:ICML15BN}, and deep residual learning \cite{cite:CVPR16DRL}. In particular, Rectifier Linear Unit (ReLU) \cite{cite:ICML10ReLU} activation function makes deep networks much easier to train and enables end-to-end learning algorithms. However, the sign activation function is so ill-defined that all the above optimization methods will fail. A very recent work, BinaryNet \cite{cite:NIPS16BinaryNet}, focuses on training deep networks with activations constrained to $+1$ or $-1$. However, the training algorithm may be hard to converge as the feed-forward pass uses the sign activation ($\operatorname{sgn}$) but the back-propagation pass uses a hard tanh ($\operatorname{Htanh}$) activation. Optimizing deep networks with sign activation remains an open problem and a key challenge for deep learning to hash.

\begin{algorithm}[htbp]
    \DontPrintSemicolon
    \KwIn{A sequence $1 = {\beta _0} < {\beta _1} <  \ldots  < {\beta _m} = \infty $}
    \For{stage $t = 0$ \KwTo $m$}{
    	Train HashNet \eqref{eqn:model} with $\operatorname{tanh}(\beta_t z)$ as activation\;
		Set converged HashNet as next stage initialization\;
    }
    \KwOut{HashNet with $\operatorname{sgn}(z)$ as activation, $\beta_m \rightarrow \infty$}
	\caption{Optimizing HashNet by Continuation}
	\label{algorithm:HashNet}
\end{algorithm}

This paper attacks the problem of non-convex optimization of deep networks with non-smooth sign activation by starting with a smoothed objective function which becomes more non-smooth as the training proceeds. It is inspired by recent studies in continuation methods \cite{cite:Book12Continuation}, which address a complex optimization problem by smoothing the original function, turning it into a different problem that is easier to optimize. By gradually reducing the amount of smoothing during the training, it results in a sequence of optimization problems converging to the original optimization problem.
Motivated by the continuation methods, we notice there exists a key relationship between the sign function and the scaled tanh function in the concept of limit in mathematics,
\begin{equation}\label{eqn:tanh}
	\mathop {\lim }\nolimits_{\beta  \to \infty } \tanh \left( {\beta z} \right) = \operatorname{sgn} \left( z \right),
\end{equation}
where $\beta>0$ is a scaling parameter. Increasing $\beta$, the scaled tanh function $\operatorname{tanh}(\beta z)$ will become more non-smooth and more saturated so that the deep networks using $\operatorname{tanh}(\beta z)$ as the activation function will be more difficult to optimize, as in Figure~\ref{fig:HashNet} (right). But fortunately, as $\beta \rightarrow \infty$, the optimization problem will converge to the original deep learning to hash problem in \eqref{eqn:model} with $\operatorname{sgn}(z)$ activation function.

Using the continuation methods, we design an optimization method for HashNet in Algorithm \ref{algorithm:HashNet}. As deep network with $\operatorname{tanh}(z)$ as the activation function can be successfully trained, we start training HashNet with $\operatorname{tanh}(\beta_t z)$ as the activation function, where $\beta_0 = 1$. For each stage $t$, after HashNet converges, we increase $\beta_t$ and train (i.e. fine-tune) HashNet by setting the converged network parameters as the initialization for training the HashNet in the next stage. By evolving $\operatorname{tanh}(\beta_t z)$ with $\beta_t \rightarrow \infty$, the network will converge to HashNet with $\operatorname{sgn}(z)$ as activation function, which can generate \emph{exactly} binary hash codes as we desire. The efficacy of continuation in Algorithm \ref{algorithm:HashNet} can be understood as multi-stage \emph{pre-training}, i.e., pre-training HashNet with $\operatorname{tanh}(\beta_t z)$ activation function is used to initialize HashNet with $\operatorname{tanh}(\beta_{t+1} z)$ activation function, which enables easier progressive training of HashNet as the network is becoming non-smooth in later stages by $\beta_t \rightarrow \infty$. Using $m=10$ we can already achieve fast convergence for training HashNet.

\subsection{Convergence Analysis}\label{sec:convergence}
We analyze that the continuation method in Algorithm~\ref{algorithm:HashNet} decreases HashNet loss~\eqref{eqn:model} in each stage and each iteration.
Let $L_{ij} = {{w_{ij}}\left( {\log \left( {1 + \exp \left( {\alpha \left\langle {{{\bm{h}}_i},{{\bm{h}}_j}} \right\rangle } \right)} \right) - \alpha {s_{ij}}\left\langle {{{\bm{h}}_i},{{\bm{h}}_j}} \right\rangle } \right)}$ and $L = \sum\nolimits_{{s_{ij}} \in \mathcal{S}} L_{ij}$, where ${\bm h}_i \in \{-1,+1\}^K$ are \emph{binary} hash codes. 
Note that when optimizing HashNet by continuation in Algorithm~\ref{algorithm:HashNet}, the network activation in each stage $t$ is $g = \tanh(\beta_t z)$, which is \emph{continuous} in nature and will only become \emph{binary} after convergence ${\beta _t} \to \infty $. Denote by $J_{ij} = {{w_{ij}}\left( {\log \left( {1 + \exp \left( {\alpha \left\langle {{{\bm{g}}_i},{{\bm{g}}_j}} \right\rangle } \right)} \right) - \alpha {s_{ij}}\left\langle {{{\bm{g}}_i},{{\bm{g}}_j}} \right\rangle } \right)}$ and $J = \sum\nolimits_{{s_{ij}} \in \mathcal{S}} J_{ij}$ the true loss we optimize in Algorithm~\ref{algorithm:HashNet}, where ${\bm g}_i \in \mathbb{R}^K$ and ${\bm h}_i = \operatorname{sgn}({\bm g}_i)$. Our results are two theorems, with proofs provided in the supplemental materials.

\begin{theorem}\label{the:stage}
The HashNet loss $L$ will not change across stages $t$ and $t+1$ with bandwidths switched from $\beta_t$ to $\beta_{t+1}$.
\end{theorem}

\begin{theorem}\label{the:converge}
Loss $L$ decreases when optimizing loss $J({\bm g})$ by the stochastic gradient descent (SGD) within each stage.
\end{theorem}

\section{Experiments}
We conduct extensive experiments to evaluate HashNet against several state-of-the-art hashing methods on three standard benchmarks. Datasets and implementations are available at \url{http://github.com/thuml/HashNet}.

\subsection{Setup}
The evaluation is conducted on three benchmark image retrieval datasets: ImageNet, NUS-WIDE and MS COCO.

\textbf{ImageNet} is a benchmark image dataset for Large Scale Visual Recognition Challenge (ILSVRC 2015) \cite{cite:ILSVRC15}. It contains over 1.2M images in the training set and 50K images in the validation set, where each image is single-labeled by one of the 1,000 categories. We randomly select 100 categories, use all the images of these categories in the training set as the database, and use all the images in the validation set as the queries; furthermore, we randomly select 100 images per category from the database as the training points.

\textbf{NUS-WIDE}\footnote{\scriptsize\url{http://lms.comp.nus.edu.sg/research/NUS-WIDE.htm}} \cite{cite:CIVR09NusWide} is a public Web image dataset which contains 269,648 images downloaded from \url{Flickr.com}. Each image is manually annotated by some of the 81 ground truth concepts (categories) for evaluating retrieval models. We follow similar experimental protocols as DHN \cite{cite:AAAI16DHN} and randomly sample 5,000 images as queries, with the remaining images used as the database; furthermore, we randomly sample 10,000 images from the database as training points.

\textbf{MS COCO}\footnote{\scriptsize\url{http://mscoco.org}} \cite{cite:MSCOCO} is an image recognition, segmentation, and captioning dataset. The current release contains 82,783 training images and 40,504 validation images, where each image is labeled by some of the 80 categories. After pruning images with no category information, we obtain 12,2218 images by combining the training and validation images. We randomly sample 5,000 images as queries, with the rest images used as the database; furthermore, we randomly sample 10,000 images from the database as training points.

Following standard evaluation protocol as previous work \cite{cite:AAAI14CNNH,cite:CVPR15DNNH,cite:AAAI16DHN}, the similarity information for hash function learning and for ground-truth evaluation is constructed from image labels: if two images $i$ and $j$ share at least one label, they are similar and  $s_{ij}=1$; otherwise, they are dissimilar and $s_{ij}=0$. Note that, although we use the image labels to construct the similarity information, our proposed HashNet can learn hash codes when only the similarity information is available. By constructing the training data in this way, the ratio between the number of  dissimilar pairs and the number of similar pairs is roughly 100, 5, and 1 for ImageNet, NUS-WIDE, and MS COCO, respectively. These datasets exhibit the data imbalance phenomenon and can be used to evaluate different hashing methods under data imbalance scenario.

We compare retrieval performance of \textbf{HashNet} with ten classical or state-of-the-art hashing methods: unsupervised methods \textbf{LSH} \cite{cite:VLDB99LSH}, \textbf{SH} \cite{cite:NIPS09SH}, \textbf{ITQ} \cite{cite:CVPR11ITQ}, supervised shallow methods \textbf{BRE} \cite{cite:NIPS09BRE}, \textbf{KSH} \cite{cite:CVPR12KSH}, \textbf{ITQ-CCA} \cite{cite:CVPR11ITQ}, \textbf{SDH} \cite{cite:CVPR15SDH}, and supervised deep methods  \textbf{CNNH} \cite{cite:AAAI14CNNH}, \textbf{DNNH} \cite{cite:CVPR15DNNH}, \textbf{DHN} \cite{cite:AAAI16DHN}. 
We evaluate retrieval quality based on five standard evaluation metrics:  Mean Average Precision (MAP), Precision-Recall curves (PR), Precision curves within Hamming distance 2 (P@H=2), Precision curves with respect to different numbers of top returned samples (P@N), and Histogram of learned codes without binarization. For fair comparison, all methods use identical training and test sets. We adopt MAP@1000 for ImageNet as each category has 1,300 images, and adopt MAP@5000 for the other datasets \cite{cite:AAAI16DHN}.
 
For shallow hashing methods, we use DeCAF$_7$ features \cite{cite:ICML14DeCAF} as input. For deep hashing methods, we use raw images as input. We adopt the AlexNet architecture \cite{cite:NIPS12CNN} for all deep hashing methods, and implement HashNet based on the \textbf{Caffe} framework \cite{cite:MM14Caffe}. We fine-tune convolutional layers $conv1$--$conv5$ and fully-connected layers $fc6$--$fc7$ copied from the AlexNet model pre-trained on ImageNet 2012 and train the hash layer $fch$, all through back-propagation. As the $fch$ layer is trained from scratch, we set its learning rate to be 10 times that of the lower layers. We use mini-batch stochastic gradient descent (SGD) with 0.9 momentum and the learning rate annealing strategy implemented in Caffe, and cross-validate the learning rate from $10^{-5}$ to $10^{-3}$ with a multiplicative step-size ${10}^{\frac{1}{2}}$. We fix the mini-batch size of images as $256$ and the weight decay parameter as $0.0005$.

\begin{table*}[tb]
    \centering 
    \addtolength{\tabcolsep}{-2pt}
    \caption{Mean Average Precision (MAP) of Hamming Ranking for Different Number of Bits on the Three Image Datasets}
    \label{table:MAP}
    \begin{tabular}{c|cccc|cccc|cccc}
        \Xhline{1.0pt}
        \multirow{2}{30pt}{\centering Method} & \multicolumn{4}{c|}{ImageNet} & \multicolumn{4}{c|}{NUS-WIDE} & \multicolumn{4}{c}{MS COCO} \\
        \cline{2-13}
        & 16 bits & 32 bits  & 48 bits  & 64 bits  & 16 bits & 32 bits  & 48 bits  & 64 bits  & 16 bits & 32 bits  & 48 bits  & 64 bits \\
        \hline
        HashNet & \textbf{0.5059} & \textbf{0.6306} & \textbf{0.6633} & \textbf{0.6835} & \textbf{0.6623} & \textbf{0.6988} & \textbf{0.7114} & \textbf{0.7163} & \textbf{0.6873} & \textbf{0.7184} & \textbf{0.7301} & \textbf{0.7362} \\
        DHN \cite{cite:AAAI16DHN} & 0.3106 & \underline{0.4717} & 0.5419 & 0.5732 & \underline{0.6374} & \underline{0.6637} & \underline{0.6692} & \underline{0.6714} & \underline{0.6774} & \underline{0.7013} & \underline{0.6948} & \underline{0.6944} \\
        DNNH \cite{cite:CVPR15DNNH} & 0.2903 & 0.4605 & 0.5301 & 0.5645 & 0.5976 & 0.6158 & 0.6345 & 0.6388 & 0.5932 & 0.6034 & 0.6045 & 0.6099 \\
        CNNH \cite{cite:AAAI14CNNH} & 0.2812 & 0.4498 & 0.5245 & 0.5538 & 0.5696 & 0.5827 & 0.5926 & 0.5996 & 0.5642 & 0.5744 & 0.5711 & 0.5671 \\
        SDH \cite{cite:CVPR15SDH} & 0.2985 & 0.4551 & \underline{0.5549} & \underline{0.5852} & 0.4756 & 0.5545 & 0.5786 & 0.5812 & 0.5545 & 0.5642 & 0.5723 & 0.5799 \\
        KSH \cite{cite:CVPR12KSH} & 0.1599 & 0.2976 & 0.3422 & 0.3943 & 0.3561 & 0.3327 & 0.3124 & 0.3368 & 0.5212 & 0.5343 & 0.5343 & 0.5361 \\
        ITQ-CCA \cite{cite:CVPR11ITQ} & 0.2659 & 0.4362 & 0.5479 & 0.5764 & 0.4598 & 0.4052 & 0.3732 & 0.3467 & 0.5659 & 0.5624 & 0.5297 & 0.5019 \\
        ITQ \cite{cite:CVPR11ITQ} & \underline{0.3255} & 0.4620 & 0.5170 & 0.5520 & 0.5086 & 0.5425 & 0.5580 & 0.5611 & 0.5818 & 0.6243 & 0.6460 & 0.6574 \\
        BRE \cite{cite:NIPS09BRE} & 0.0628 & 0.2525 & 0.3300 & 0.3578 & 0.5027 & 0.5290 & 0.5475 & 0.5546 & 0.5920 & 0.6224 & 0.6300 & 0.6336 \\
        SH \cite{cite:NIPS09SH} & 0.2066 & 0.3280 & 0.3951 & 0.4191 & 0.4058 & 0.4209 & 0.4211 & 0.4104 & 0.4951 & 0.5071 & 0.5099 & 0.5101 \\
        LSH \cite{cite:VLDB99LSH} & 0.1007 & 0.2350 & 0.3121 & 0.3596 & 0.3283 & 0.4227 & 0.4333 & 0.5009 & 0.4592 & 0.4856 & 0.5440 & 0.5849 \\
        \Xhline{1.0pt}
    \end{tabular}
\end{table*}

\begin{figure*}[!phtb]
    \centering
    \subfigure[Precision within Hamming radius 2]{
        \includegraphics[width=0.28\textwidth]{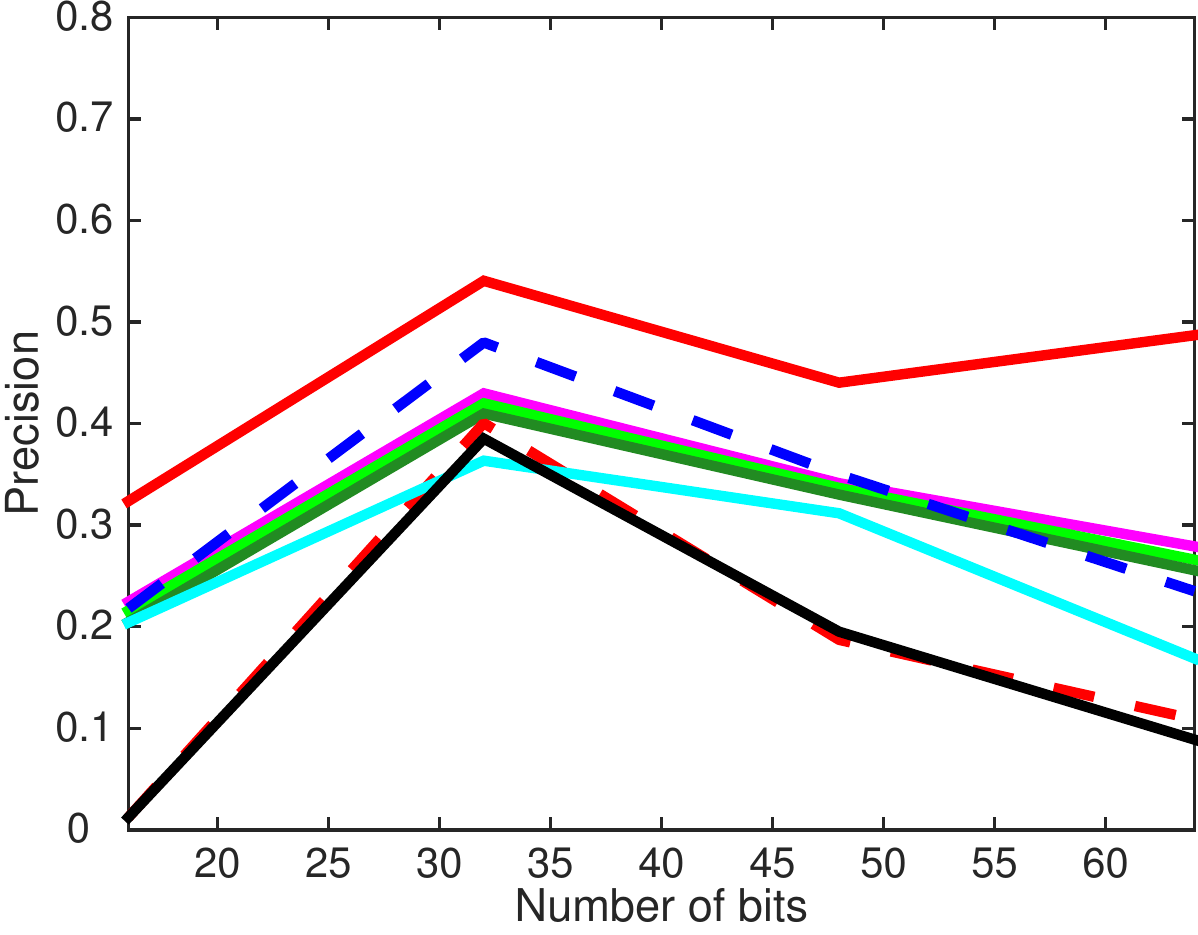}
        \label{fig:ham_imagenet}
    }
		\hfil
    \subfigure[Precision-recall curve @ 64 bits]{
        \includegraphics[width=0.28\textwidth]{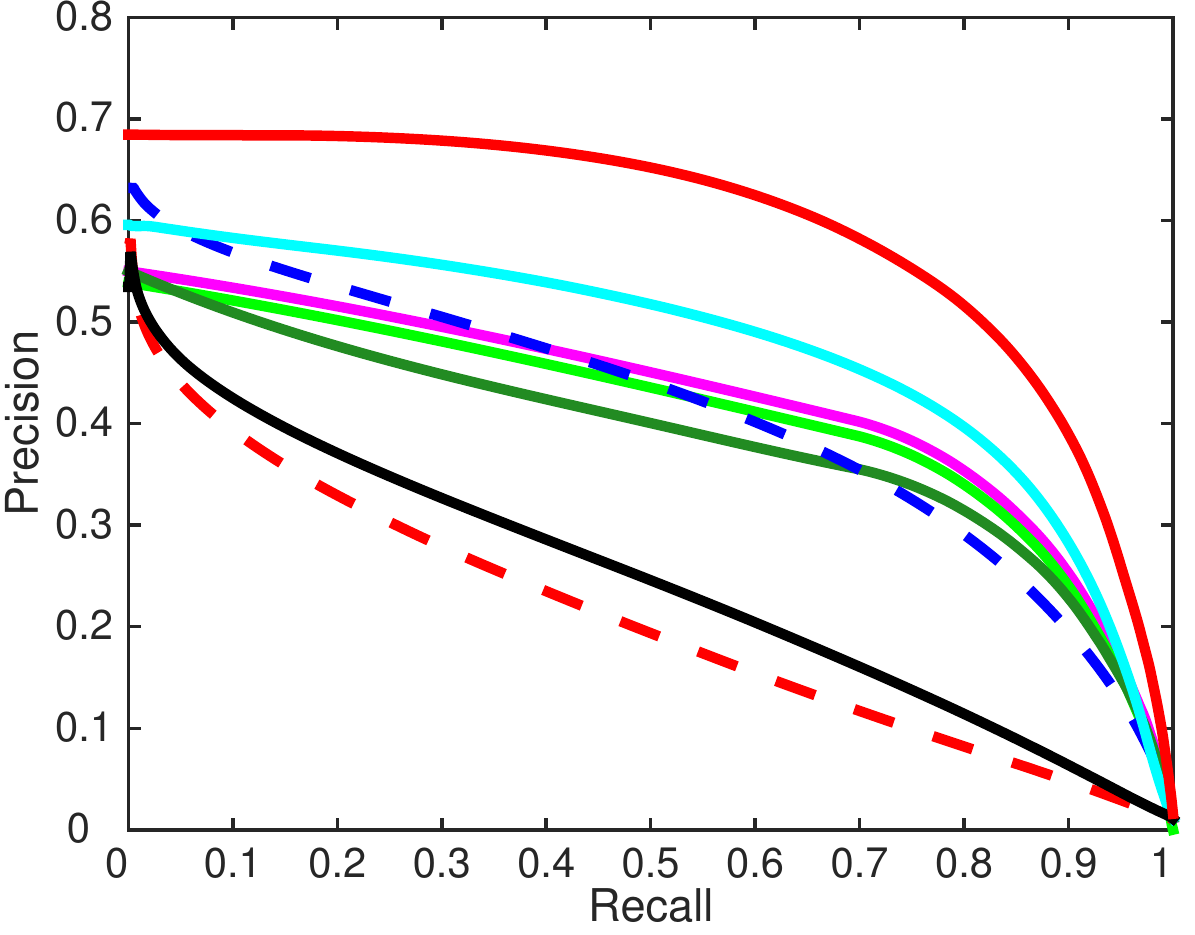}
        \label{fig:pr_imagenet}
    }
		\hfil
    \subfigure[Precision curve w.r.t. top-$N$ @ 64 bits]{
        \includegraphics[width=0.36\textwidth]{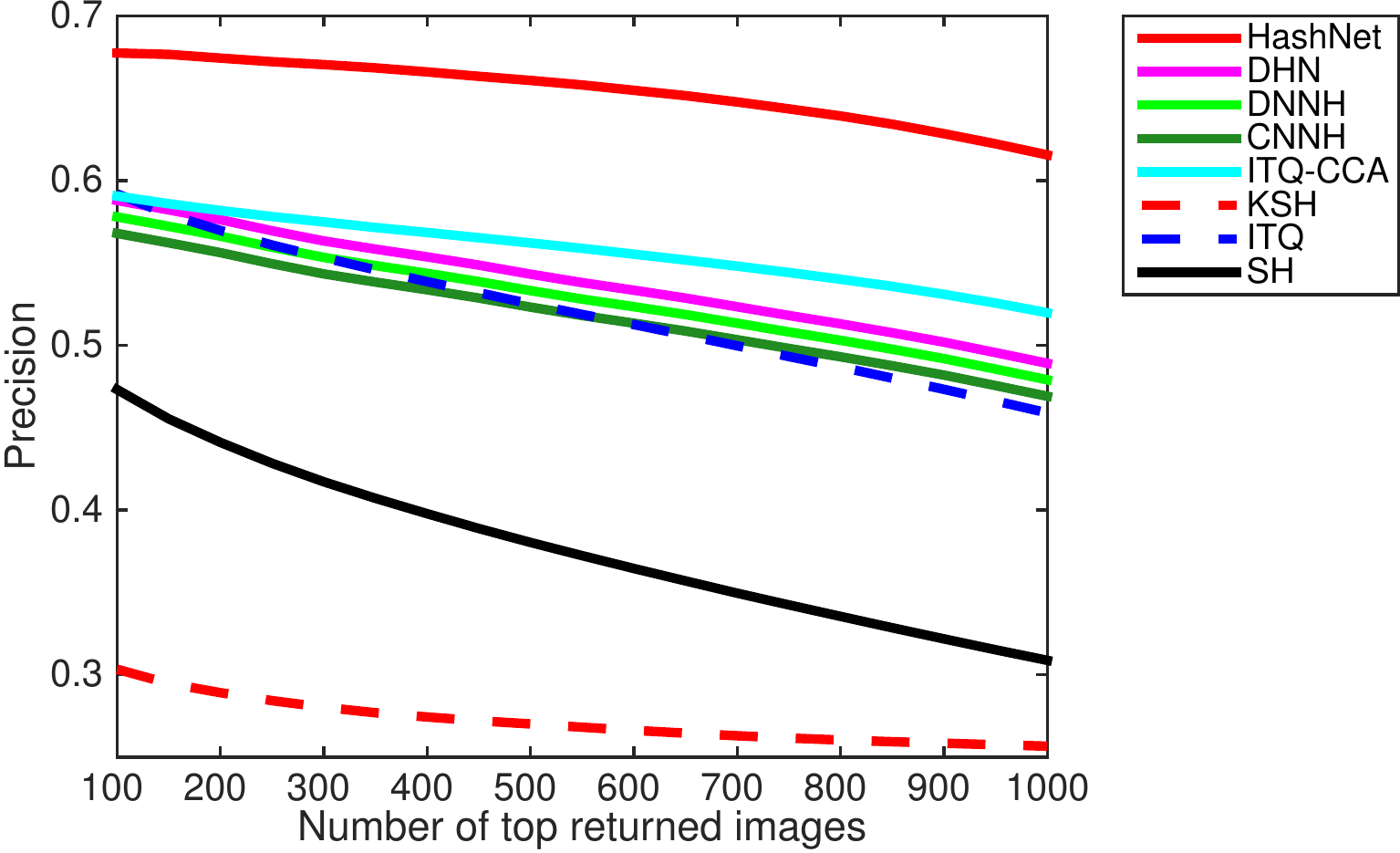}
        \label{fig:prec_imagenet}
    }
    \caption{The experimental results of HashNet and comparison methods on the ImageNet dataset under three evaluation metrics.}
    \label{fig:imagenet}
\end{figure*}

\begin{figure*}[!phtb]
    \centering
    \subfigure[Precision within Hamming radius 2]{
        \includegraphics[width=0.28\textwidth]{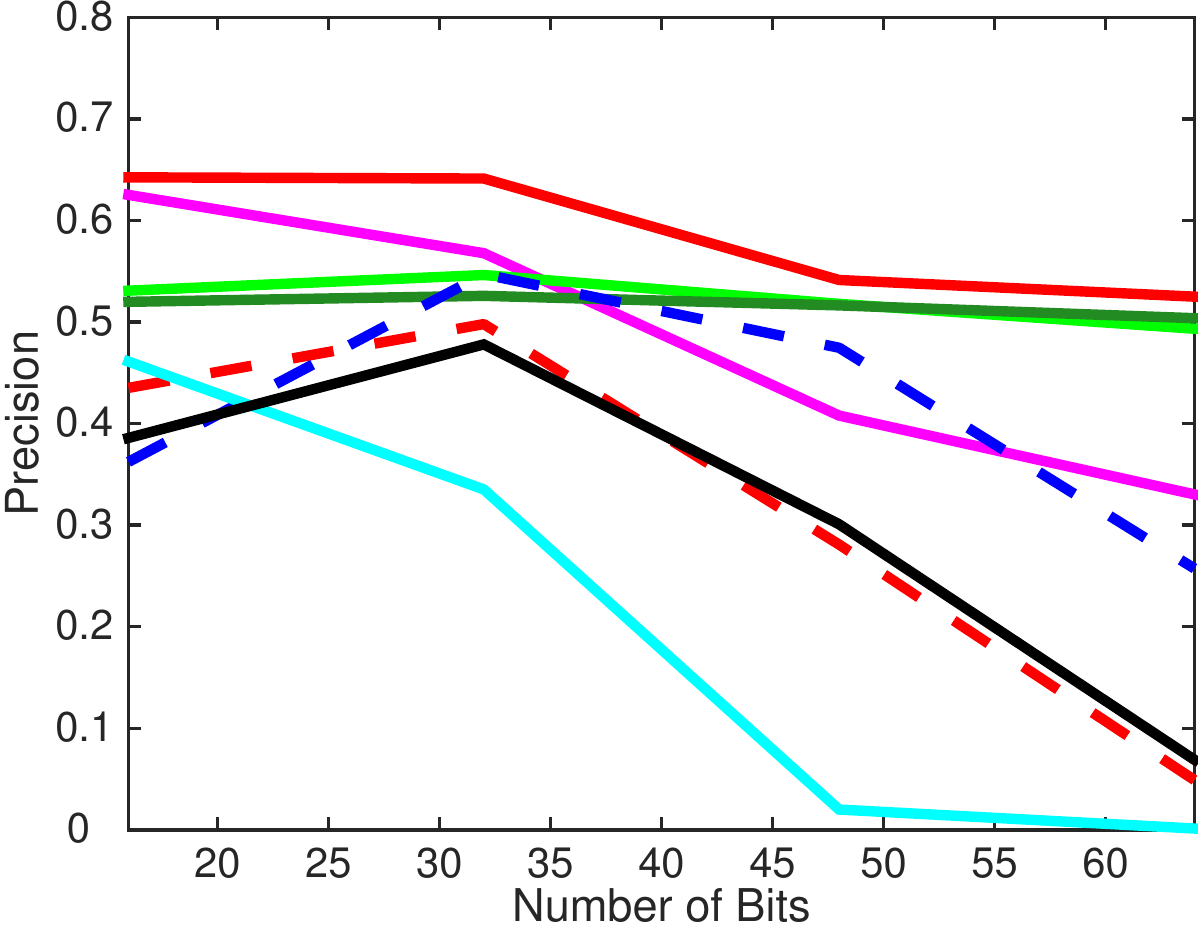}
        \label{fig:ham_nus}
    }
		\hfil
    \subfigure[Precision-recall curve @ 64 bits]{
        \includegraphics[width=0.28\textwidth]{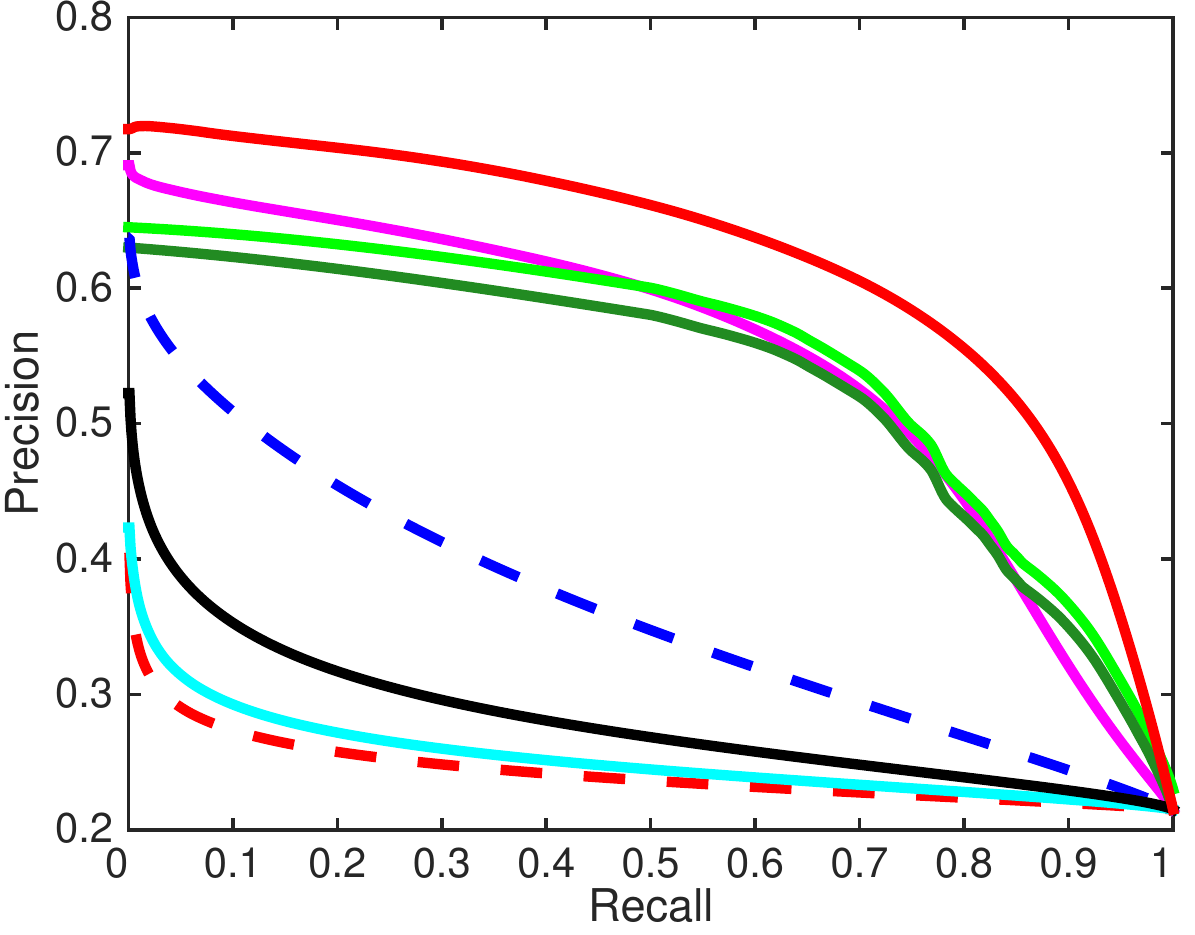}
        \label{fig:pr_nus}
    }
		\hfil
    \subfigure[Precision curve w.r.t. top-$N$ @ 64 bits]{
        \includegraphics[width=0.36\textwidth]{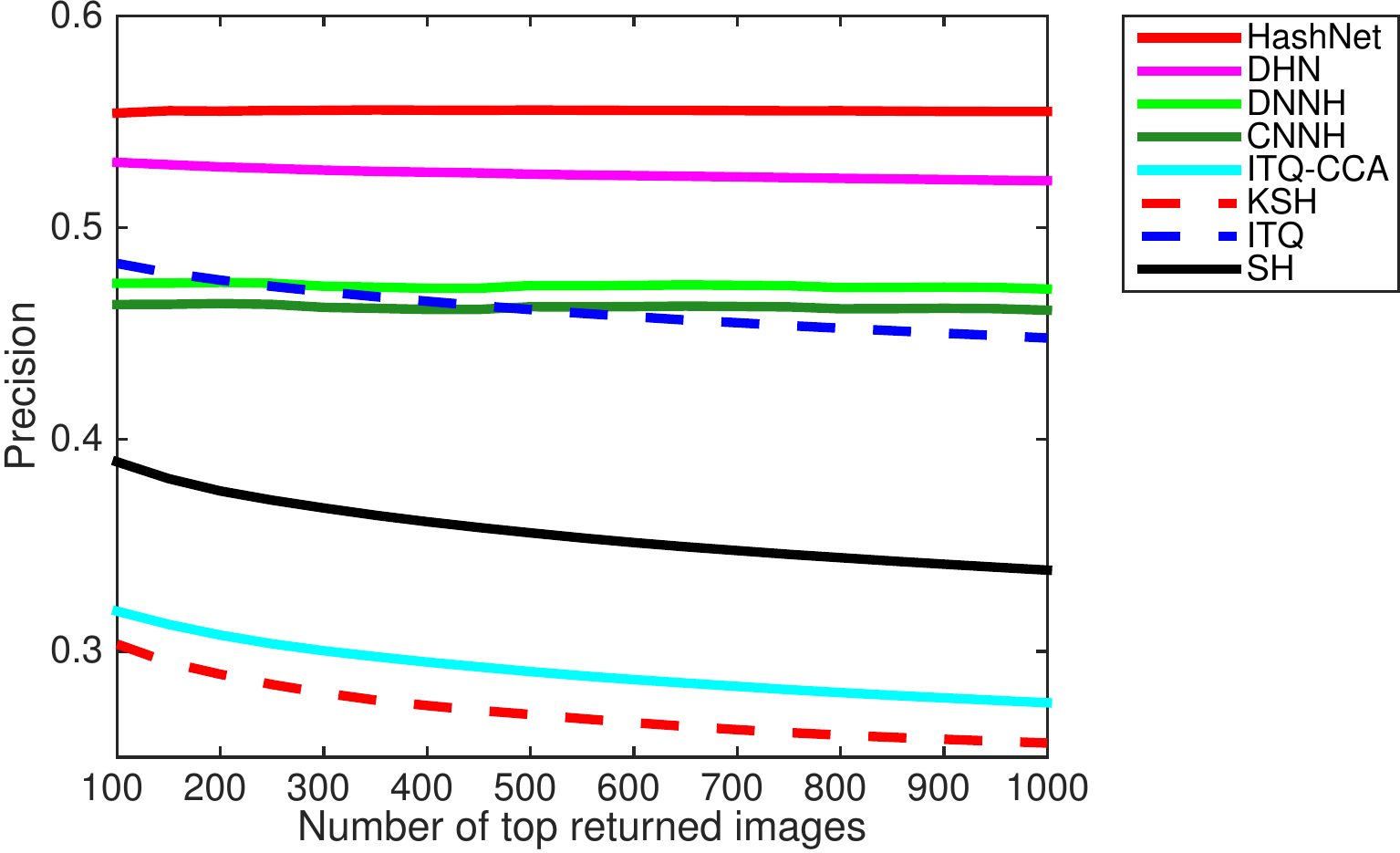}
        \label{fig:prec_nus}
    }
    \caption{The experimental results of HashNet and comparison methods on the NUS-WIDE dataset under three evaluation metrics.}
    \label{fig:nus}
\end{figure*}

\begin{figure*}[!phtb]
    \centering
    \subfigure[Precision within Hamming radius 2]{
        \includegraphics[width=0.28\textwidth]{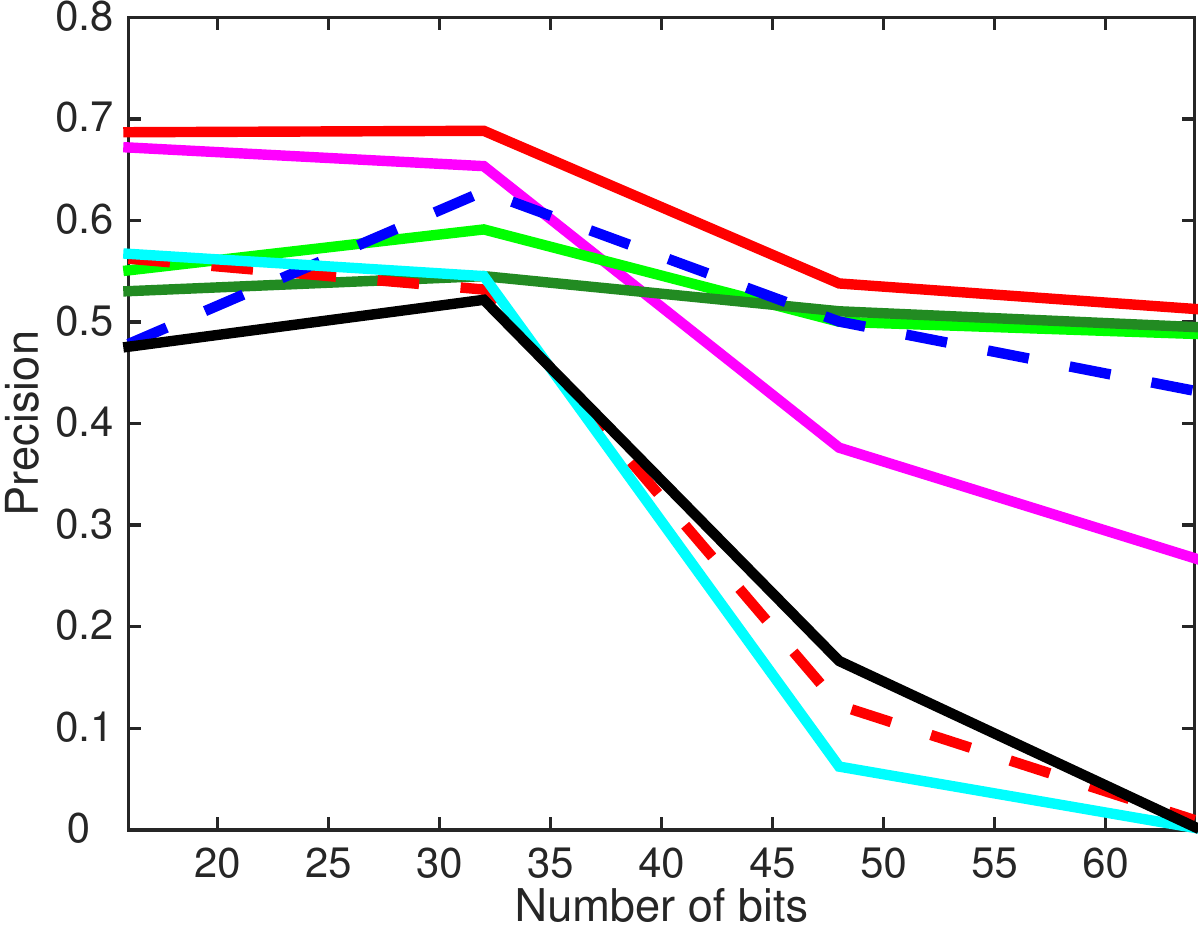}
        \label{fig:ham_coco}
    }
		\hfil
    \subfigure[Precision-recall curve @ 64 bits]{
        \includegraphics[width=0.28\textwidth]{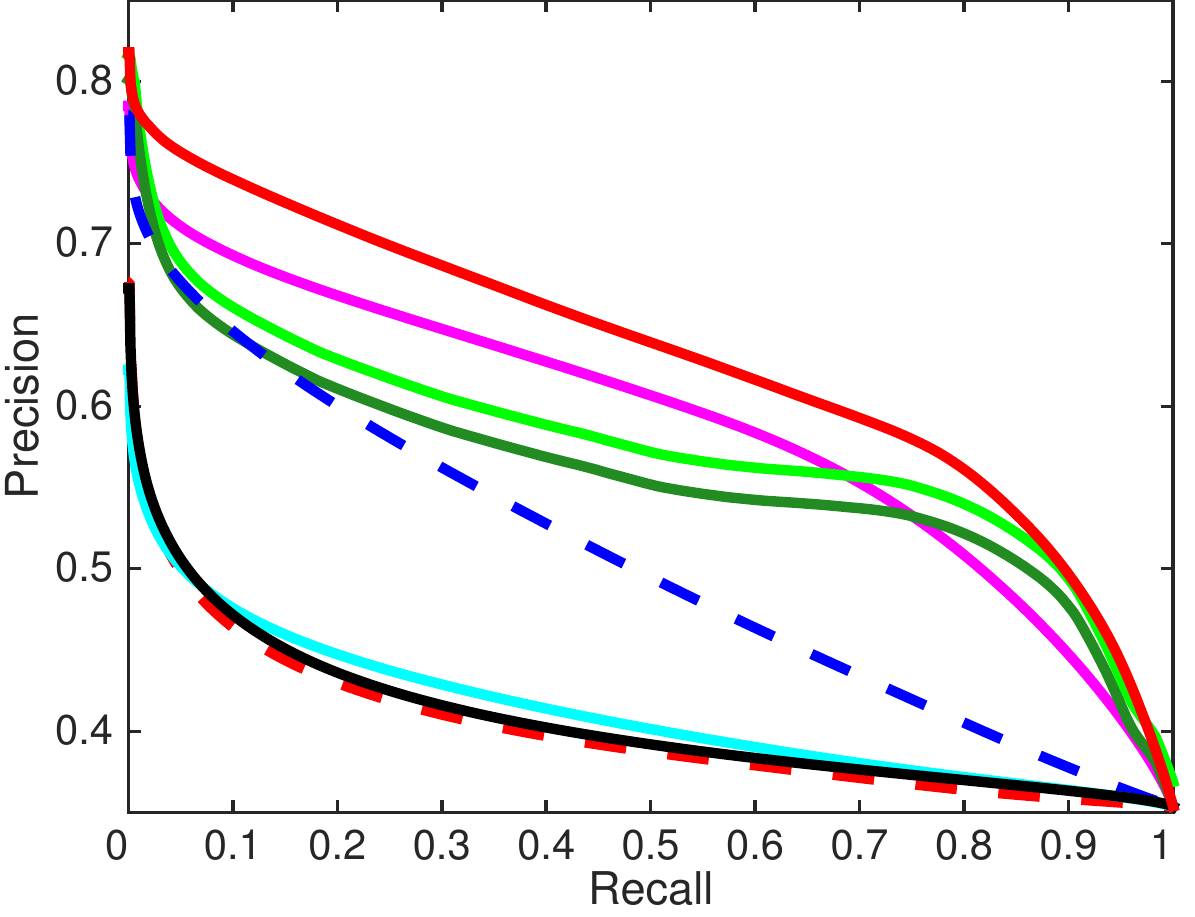}
        \label{fig:pr_coco}
    }
		\hfil
    \subfigure[Precision curve w.r.t. top-$N$ @ 64 bits]{
        \includegraphics[width=0.36\textwidth]{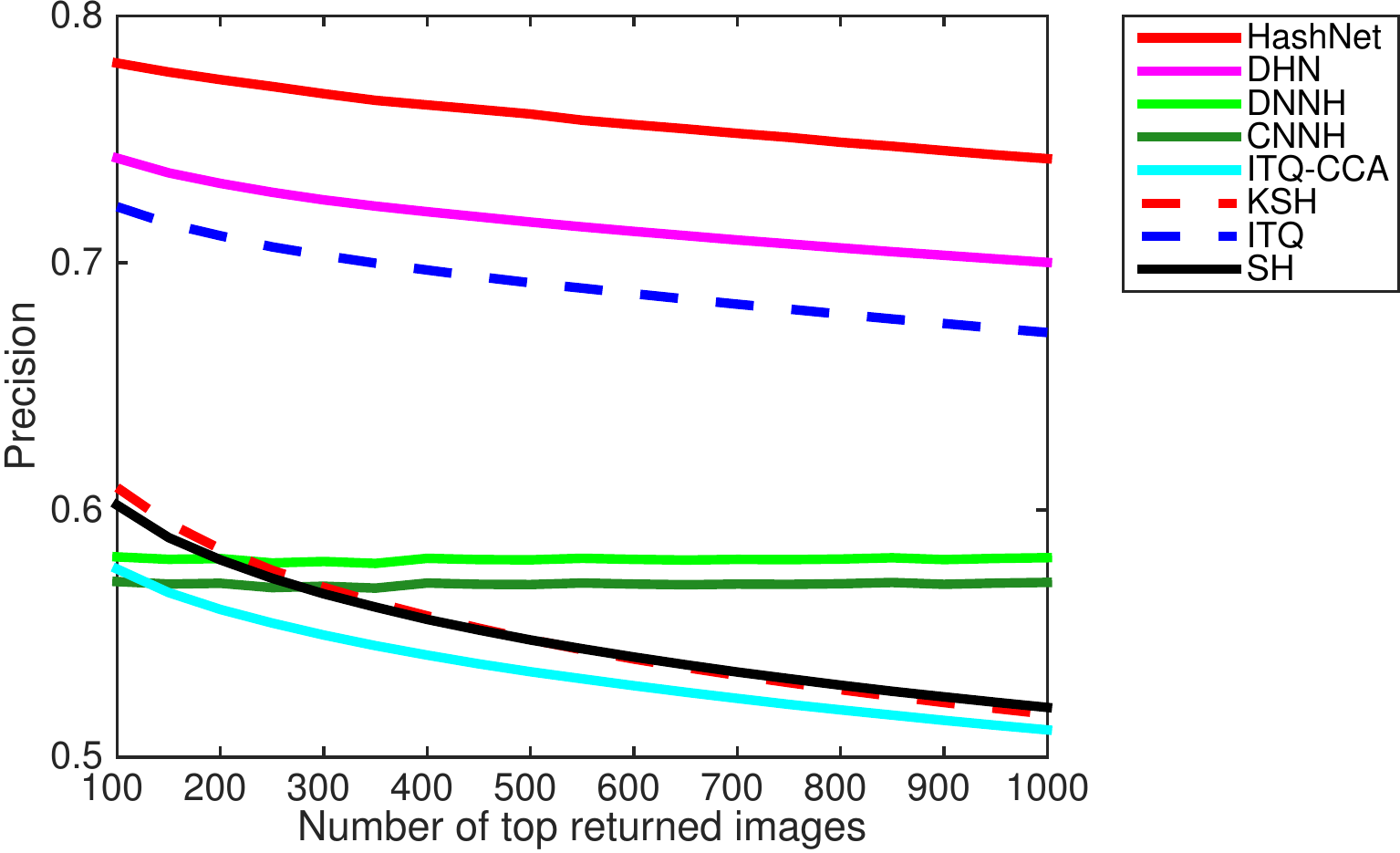}
        \label{fig:prec_coco}
    }
    \caption{The experimental results of HashNet and comparison methods on the MS COCO dataset under three evaluation metrics.}
    \label{fig:coco}
    \vspace{-10pt}
\end{figure*}

\begin{figure}[!tbp]
  \centering
  \includegraphics[width=1.0\columnwidth]{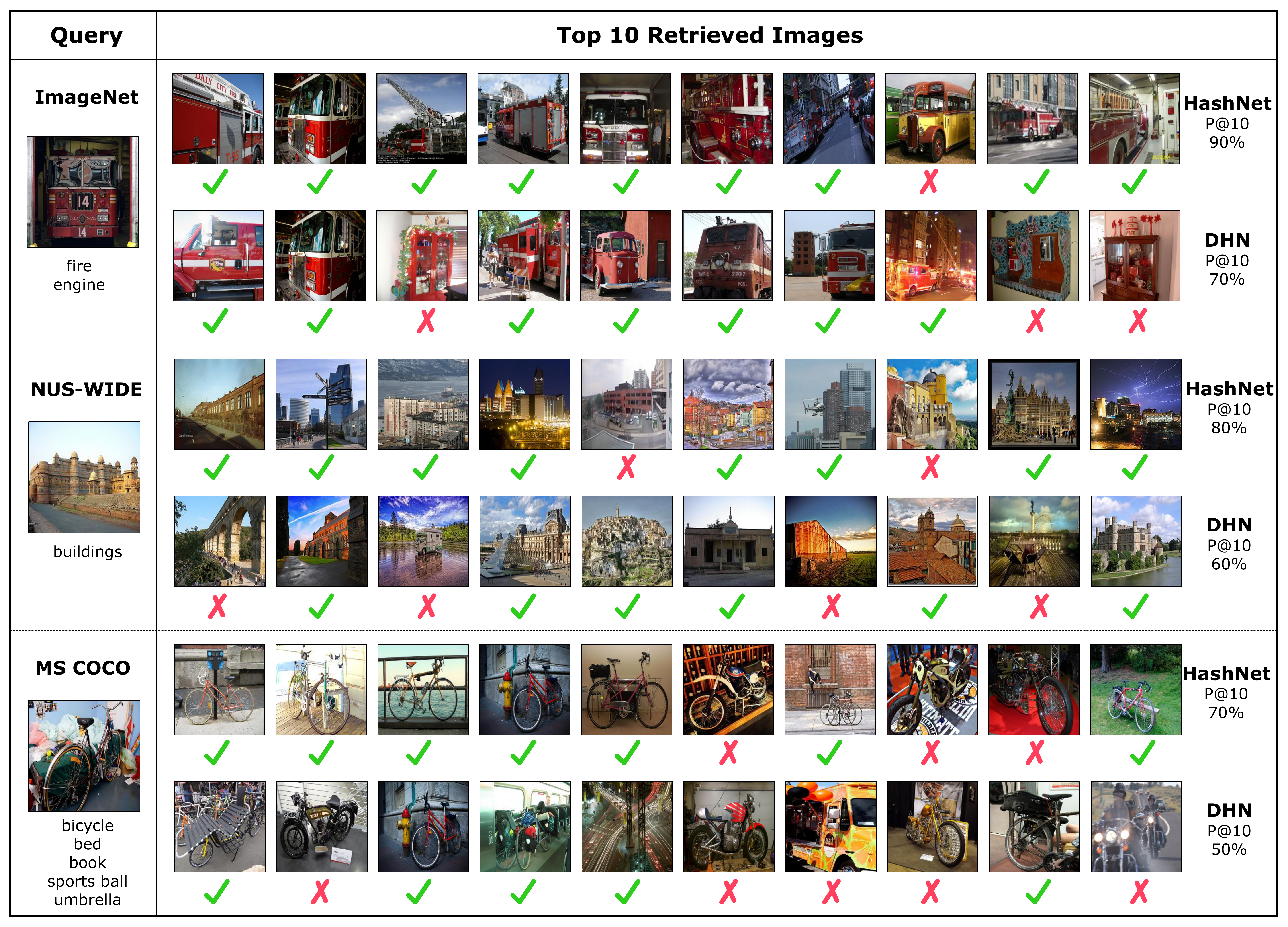}
  \caption{Examples of top 10 retrieved images and precision@10.}
   \label{fig:top10}
   \vspace{-10pt}
\end{figure}

\subsection{Results}
The Mean Average Precision (MAP) results are shown in Table \ref{table:MAP}. HashNet substantially outperforms all comparison methods. Specifically, compared to the best shallow hashing method using deep features as input, ITQ/ITQ-CCA, we achieve absolute boosts of $15.7\%$, $15.5\%$, and $9.1\%$ in average MAP for different bits on ImageNet, NUS-WIDE, and MS COCO, respectively. Compared to the state-of-the-art deep hashing method, DHN, we achieve absolute boosts of $14.6\%$, $3.7\%$, $2.9\%$ in average MAP for different bits on the three datasets, respectively. An interesting phenomenon is that the performance boost of HashNet over DHN is significantly different across the three datasets. Specifically, the performance boost on ImageNet is much larger than that on NUS-WIDE and MS COCO by about $10\%$, which is very impressive. Recall that the ratio between the number of dissimilar pairs and the number of similar pairs is roughly $100$, $5$, and $1$ for ImageNet, NUS-WIDE and MS COCO, respectively. This data imbalance problem substantially deteriorates the performance of hashing methods trained from pairwise data, including all the deep hashing methods. HashNet enhances deep learning to hash from imbalanced dataset by Weighted Maximum Likelihood (WML), which is a principled solution to tackling the data imbalance problem. This lends it the superior performance on imbalanced datasets. 

The performance in terms of Precision within Hamming radius 2 (P@H=2) is very important for efficient retrieval with binary hash codes since such Hamming ranking only requires $O(1)$ time for each query. As shown in Figures \ref{fig:ham_imagenet}, \ref{fig:ham_nus} and \ref{fig:ham_coco}, HashNet achieves the highest P@H=2 results on all three datasets. In particular, P@H=2 of HashNet with 32 bits is better than that of DHN with any bits. This validates that HashNet can learn more compact binary codes than DHN. When using longer codes, the Hamming space will become sparse and few data points fall within the Hamming ball with radius 2 \cite{cite:CVPR12MIH}. This is why most hashing methods achieve best accuracy with moderate code lengths.

The retrieval performance on the three datasets in terms of Precision-Recall curves (PR) and Precision curves with respect to different numbers of top returned samples (P@N) are shown in Figures \ref{fig:pr_imagenet}$\sim$\ref{fig:pr_coco} and Figures \ref{fig:prec_imagenet}$\sim$\ref{fig:prec_coco}, respectively. HashNet outperforms comparison methods by large margins. In particular, HashNet achieves much higher precision at lower recall levels or when the number of top results is small. This is desirable for precision-first retrieval, which is widely implemented in practical systems. As an intuitive illustration, Figure~\ref{fig:top10} shows that HashNet can yield much more relevant and user-desired retrieval results. 

Recent work \cite{cite:TIP17HSL} studies two evaluation protocols for supervised hashing: (1) supervised retrieval protocol where queries and database have identical classes and (2) zero-shot retrieval protocol where queries and database have different classes. Some supervised hashing methods perform well in one protocol but poorly in another protocol. Table~\ref{table:zeroshot} shows the MAP results on ImageNet dataset under the zero-shot retrieval protocol, where HashNet substantially outperforms DHN. Thus, HashNet works well under different protocols.

\begin{table}[!tbp]
    \centering
    \addtolength{\tabcolsep}{3pt}
    \caption{MAP on ImageNet with Zero-Shot Retrieval Protocol \cite{cite:TIP17HSL}}
    \label{table:zeroshot}
    \begin{tabular}{c|cccc}
        \Xhline{1.0pt}
        {\centering Method} & 16 bits & 32 bits  & 48 bits  & 64 bits \\
        \hline
        HashNet & \textbf{0.4411} & \textbf{0.5274} & \textbf{0.5651} & \textbf{0.5756} \\
        DHN \cite{cite:AAAI16DHN} & 0.2891 & 0.4421 & 0.5123 & 0.5342  \\
        \Xhline{1.0pt}
    \end{tabular}
    \vspace{-10pt}
\end{table}

\begin{figure}[!tbp]
    \centering
    \subfigure[HashNet]{
        \includegraphics[width=0.35\columnwidth]{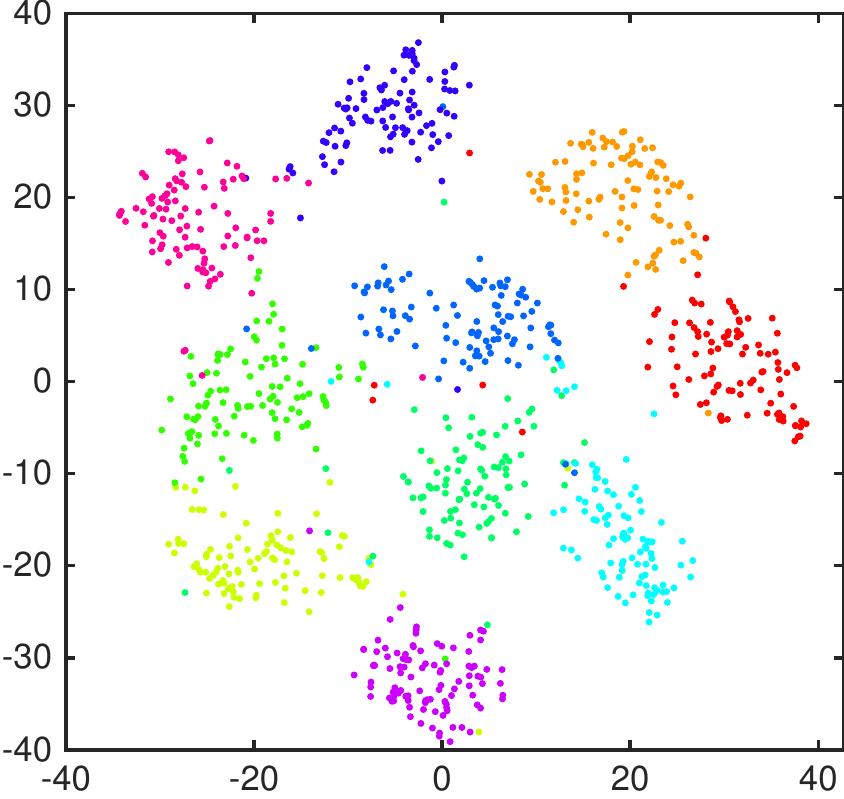}
        \label{fig:t-sne_hashnet}
    }
    \hspace{10pt}
    \subfigure[DHN]{
        \includegraphics[width=0.36\columnwidth]{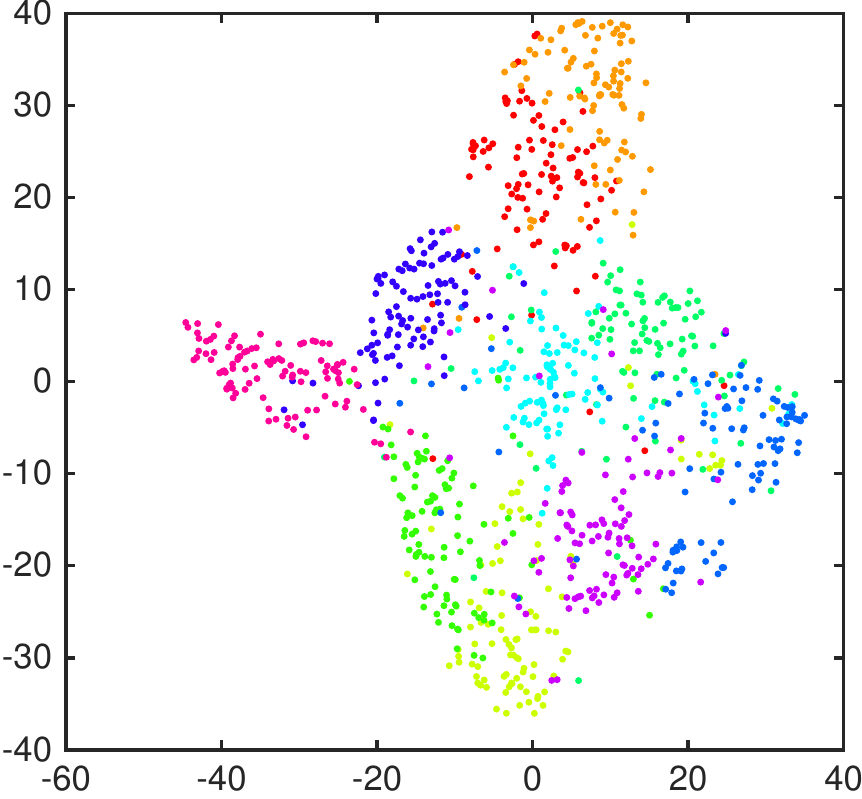}
        \label{fig:t-sne_dhn}
    }
    \caption{The t-SNE of hash codes learned by HashNet and DHN.}
    \label{fig:t-sne}
    \vspace{-10pt}
\end{figure}

\begin{table*}[htb]
    \centering
    \addtolength{\tabcolsep}{-1.7pt}
    \caption{Mean Average Precision (MAP) Results of HashNet and Its Variants, HashNet+C, HashNet-W, and HashNet-sgn on Three Datasets}
    \label{table:ablation}
    \begin{tabular}{c|cccc|cccc|cccc}
        \Xhline{1.0pt}
        \multirow{2}{30pt}{\centering Method} & \multicolumn{4}{c|}{ImageNet} & \multicolumn{4}{c|}{NUS-WIDE} & \multicolumn{4}{c}{MS COCO} \\
        \cline{2-13}
        & 16 bits & 32 bits  & 48 bits  & 64 bits & 16 bits & 32 bits  & 48 bits  & 64 bits & 16 bits & 32 bits  & 48 bits  & 64 bits \\
        \hline
        HashNet+C & \textbf{0.5059} & \textbf{0.6306} & \textbf{0.6633} & \textbf{0.6835} & \textbf{0.6646} & \textbf{0.7024} & \textbf{0.7209} & \textbf{0.7259} & \textbf{0.6876} & \textbf{0.7261} & \textbf{0.7371} & \textbf{0.7419} \\
        HashNet & \textbf{0.5059} & \textbf{0.6306} & \textbf{0.6633} & \textbf{0.6835} & \underline{0.6623} & \underline{0.6988} & \underline{0.7114} & \underline{0.7163} & \underline{0.6873} & \underline{0.7184} & \underline{0.7301} & \underline{0.7362} \\
        HashNet-W & 0.3350 & 0.4852 & 0.5668 & 0.5992 & 0.6400 & 0.6638 & 0.6788 & 0.6933 & 0.6853 & 0.7174 & 0.7297 & 0.7348  \\
        HashNet-sgn & \underline{0.4249} & \underline{0.5450} & \underline{0.5828} & \underline{0.6061} & 0.6603 & 0.6770 & 0.6921 & 0.7020 & 0.6449 & 0.6891 & 0.7056 & 0.7138 \\
        \Xhline{1.0pt}
    \end{tabular}
    \normalsize
    \vspace{-10pt}
\end{table*}

\subsection{Empirical Analysis}
\textbf{Visualization of Hash Codes:}
We visualize the t-SNE \cite{cite:ICML14DeCAF} of hash codes generated by HashNet and DHN on ImageNet in Figure~\ref{fig:t-sne} (for ease of visualization, we sample 10 categories). We observe that the hash codes generated by HashNet show clear discriminative structures in that different categories are well separated, while the hash codes generated by DHN do not show such discriminative structures. This suggests that HashNet can learn more discriminative hash codes than DHN for more effective similarity retrieval.

\textbf{Ablation Study:}
We go deeper with the efficacy of the weighted maximum likelihood and continuation methods. We investigate three variants of HashNet: (1) \textbf{HashNet+C}, variant using continuous similarity ${c_{ij}} = \frac{{{{\bm{y}}_i} \cap {{\bm{y}}_j}}}{{{{\bm{y}}_i} \cup {{\bm{y}}_j}}}$ when image labels are given; (2) \textbf{HashNet-W}, variant using maximum likelihood instead of weighted maximum likelihood, i.e. $w_{ij} = 1$; (3) \textbf{HashNet-sgn}, variant using $\operatorname{tanh}()$ instead of $\operatorname{sgn}()$ as activation function to generate continuous codes and requiring a separated binarization step to generate hash codes. We compare results of these variants in Table~\ref{table:ablation}.

By weighted maximum likelihood estimation, HashNet outperforms HashNet-W by substantially large margins of $12.4\%$, $2.8\%$ and $0.1\%$ in average MAP for different bits on ImageNet, NUS-WIDE and MS COCO, respectively. The standard maximum likelihood estimation has been widely adopted in previous work \cite{cite:AAAI14CNNH,cite:AAAI16DHN}. However, this estimation does not account for the data imbalance, and may suffer from performance drop when training data is highly imbalanced (e.g. ImageNet). In contrast, the proposed weighted maximum likelihood estimation \eqref{eqn:WML} is a principled solution to tackling the data imbalance problem by weighting the training pairs according to the importance of misclassifying that pair. Recall that MS COCO is a balanced dataset, hence HashNet and HashNet-W may yield similar MAP results. By further considering continuous similarity (${c_{ij}} = \frac{{{{\bm{y}}_i} \cap {{\bm{y}}_j}}}{{{{\bm{y}}_i} \cup {{\bm{y}}_j}}}$), HashNet+C achieves even better accuracy than HashNet.

By training HashNet with continuation, HashNet outperforms HashNet-sgn by substantial margins of $8.1\%$, $1.4\%$ and $3.0\%$ in average MAP on ImageNet, NUS-WIDE, and MS COCO, respectively. Due to the ill-posed gradient problem, existing deep hashing methods cannot learn exactly binary hash codes using $\operatorname{sgn}()$ as activation function. Instead, they need to use surrogate functions of $\operatorname{sgn}()$, e.g. $\operatorname{tanh}()$, as the activation function and learn continuous codes, which require a separated binarization step to generate hash codes. 
The proposed continuation method is a principled solution to deep learning to hash with $\operatorname{sgn}()$ as activation function, which learn lossless binary hash codes for accurate retrieval.

\begin{figure}[!tbp]
  \centering
  \subfigure[ImageNet]{
    \includegraphics[width=0.3\columnwidth]{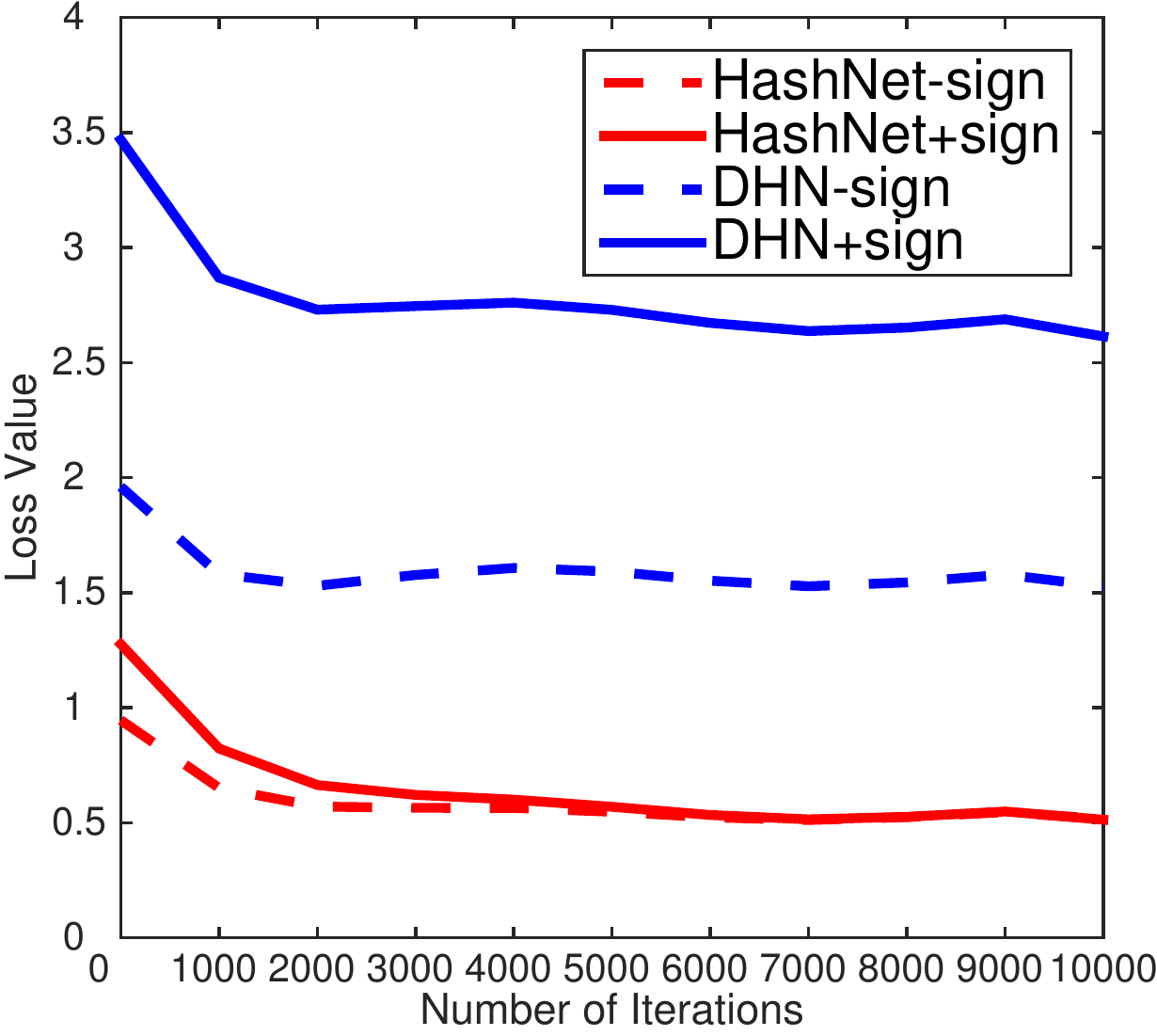}
    \label{fig:loss_imagenet}
  }
  \subfigure[NUS-WIDE]{
    \includegraphics[width=0.3\columnwidth]{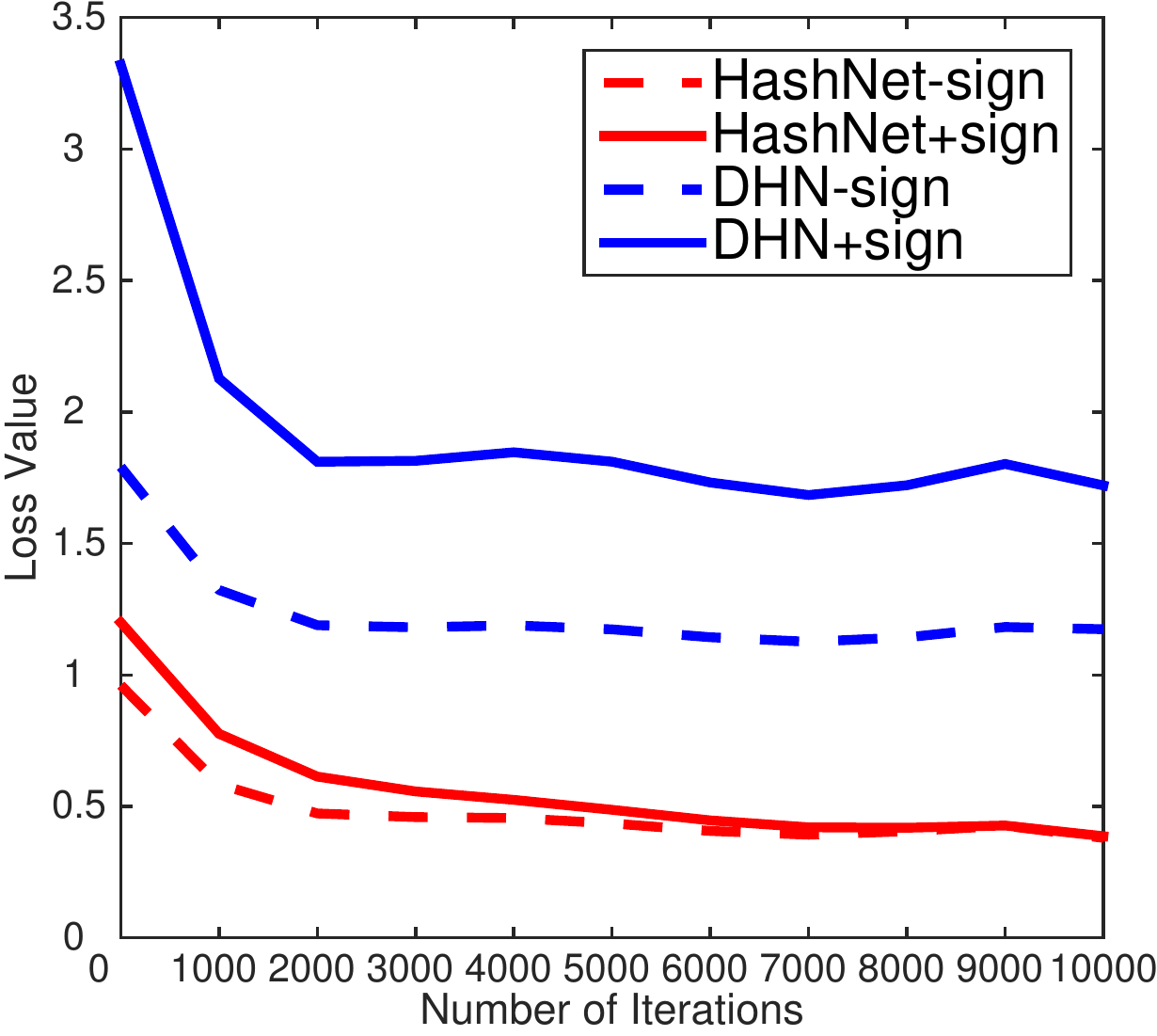}
    \label{fig:loss_nus}
  }
  \subfigure[COCO]{
    \includegraphics[width=0.3\columnwidth]{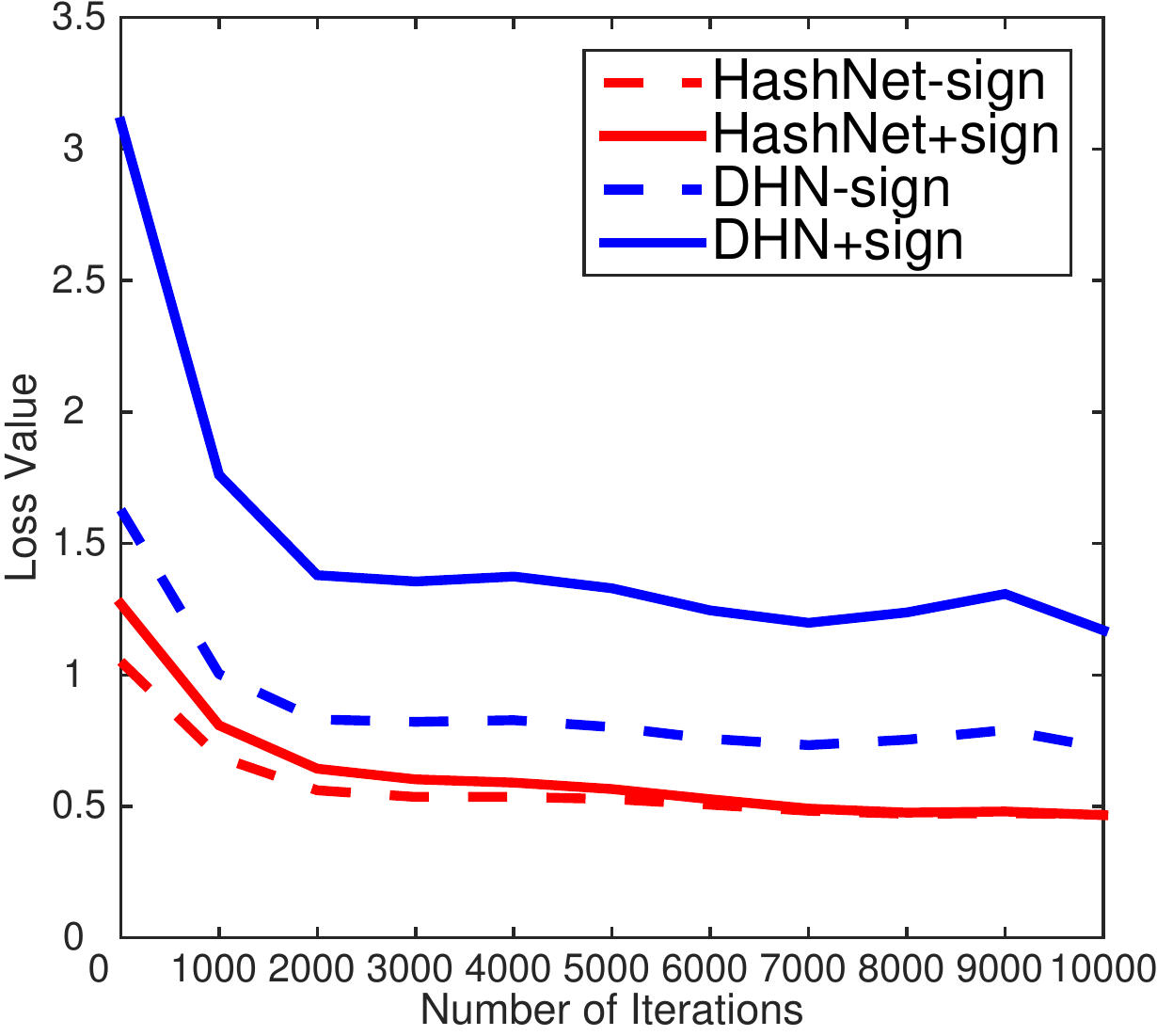}
    \label{fig:loss_coco}
  }
  \caption{Losses of HashNet and DHN through training process.}
  \label{fig:loss_value}
  \vspace{-10pt}
\end{figure}

\textbf{Loss Value Through Training Process:} We compare the change of loss values of HashNet and DHN through the training process on ImageNet, NUS-WIDE and MSCOCO. We display the loss values before (-sign) and after (+sign) binarization, i.e. $J({\bm g})$ and $L({\bm h})$. Figure~\ref{fig:loss_value} reveals three important observations: \textbf{(a)} Both methods converge in terms of the loss values before and after binarization, which validates the convergence analysis in Section~\ref{sec:convergence}. \textbf{(b)} HashNet converges with a much smaller training loss than DHN both before and after binarization, which implies that HashNet can preserve the similarity relationship in \emph{Hamming} space much better than DHN. \textbf{(c)} The two loss curves of HashNet before and after binarization become close to each other and overlap completely when convergence. This shows that the continuation method enables HashNet to approach the true loss defined on the exactly binary codes without continuous relaxation. But there is a large gap between two loss curves of DHN, implying that DHN and similar methods \cite{cite:CVPR15SDH,cite:IJCAI16DPSH,cite:CVPR2016DSH} cannot learn exactly binary codes by minimizing quantization error of codes before and after binarization.

\begin{figure}[!tbp]
  \centering
  \subfigure[ImageNet]{
    \includegraphics[width=0.3\columnwidth]{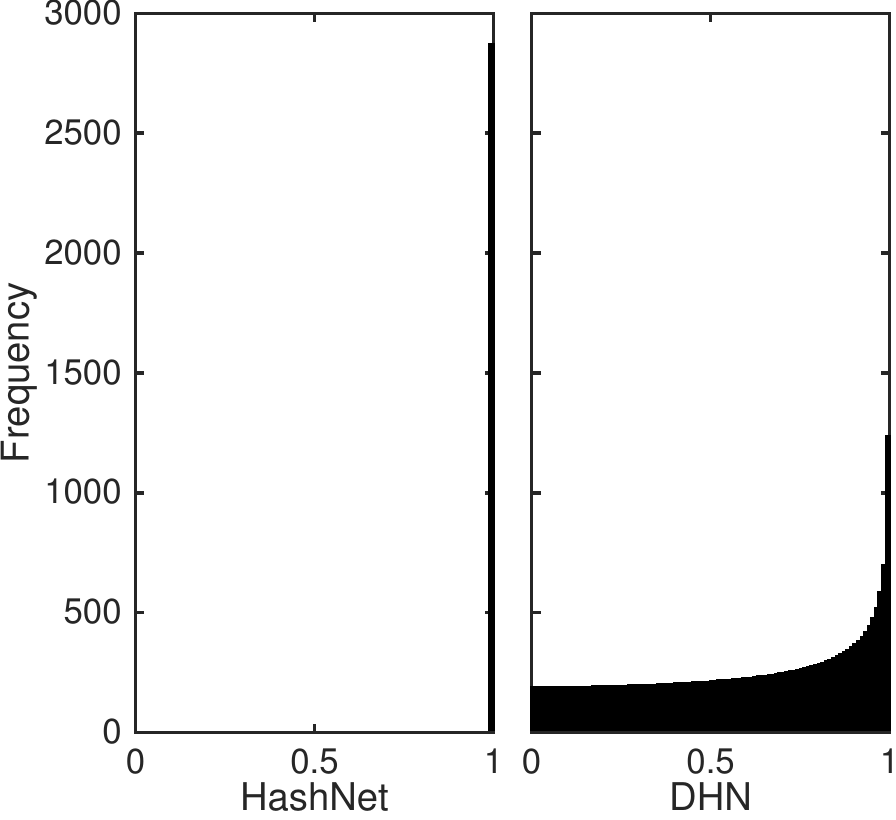}
    \label{fig:hist_imagenet}
  }
  \subfigure[NUS-WIDE]{
    \includegraphics[width=0.3\columnwidth]{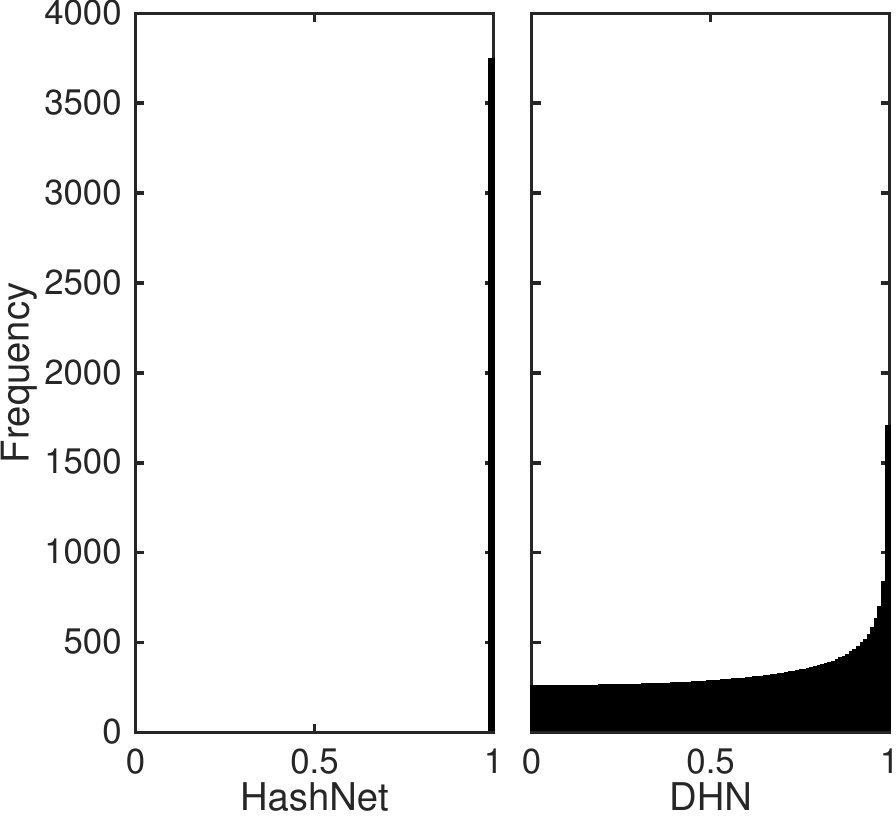}
    \label{fig:hist_nus}
  }
  \subfigure[COCO]{
    \includegraphics[width=0.3\columnwidth]{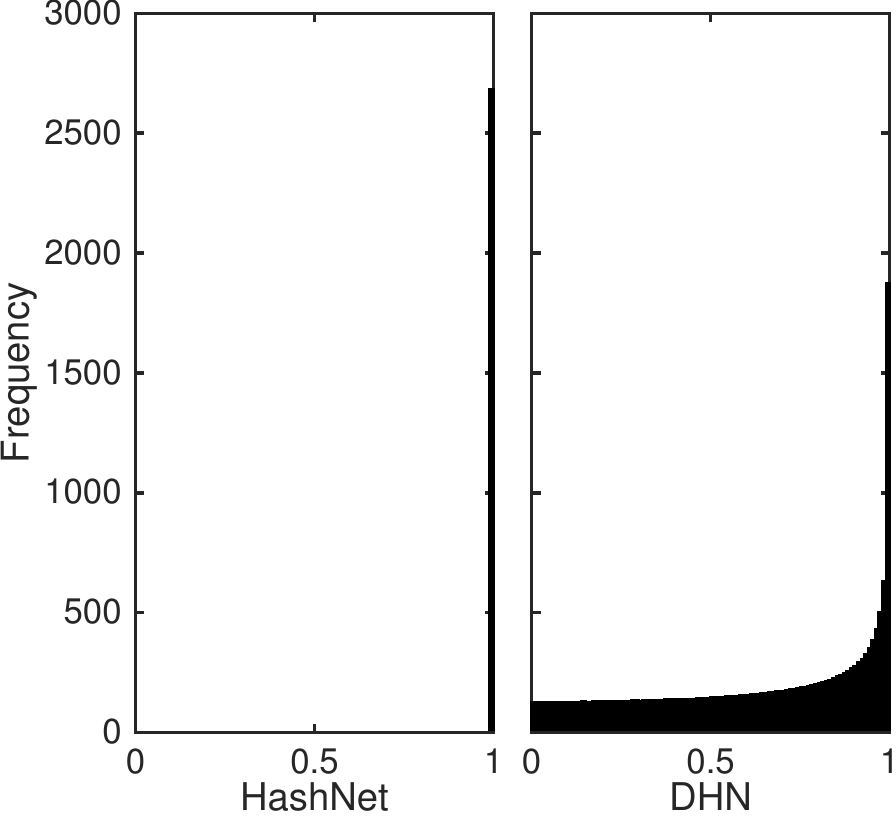}
    \label{fig:hist_coco}
  }
  \caption{Histogram of non-binarized codes of HashNet and DHN.}
  \label{fig:hist}
  \vspace{-10pt}
\end{figure}

\textbf{Histogram of Codes Without Binarization:}
As discussed previously, the proposed HashNet can learn exactly binary hash codes while previous deep hashing methods can only learn continuous codes and generate binary hash codes by post-step sign thresholding. To verify this key property, we plot the histograms of codes learned by HashNet and DHN on the three datasets without post-step binarization. The histograms can be plotted by evenly dividing $[0,1]$ into 100 bins, and calculating the frequency of codes falling into each bin. To make the histograms more readable, we show absolute code values ($x$-axis) and squared root of frequency ($y$-axis). Histograms in Figure~\ref{fig:hist} show that DHN can only generate continuous codes spanning across the whole range of $[0,1]$. This implies that if we quantize these continuous codes into binary hash codes (taking values in $\{-1,1\}$) in a post-step, we may suffer from large quantization error especially for the codes near zero. On the contrary, the codes of HashNet without binarization are already exactly binary.

\section{Conclusion}
This paper addressed deep learning to hash from imbalanced similarity data by the continuation method. The proposed HashNet can learn exactly binary hash codes by optimizing a novel weighted pairwise cross-entropy loss function in deep convolutional neural networks. HashNet can be effectively trained by the proposed multi-stage pre-training algorithm carefully crafted from the continuation method. Comprehensive empirical evidence shows that HashNet can generate exactly binary hash codes and yield state-of-the-art multimedia retrieval performance on standard benchmarks.

\section{Acknowledgments}
This work was supported by the National Key R\&D Program of China (No. 2016YFB1000701), the National Natural Science Foundation of China (No. 61502265, 61325008, and 71690231), the National Sci.\&Tech. Supporting Program (2015BAF32B01), and the Tsinghua TNList Projects.

\section{Supplemental Material:\\HashNet: Deep Learning to Hash by Continuation}

\subsection{Convergence Analysis}
We briefly analyze that the continuation optimization in Algorithm~\ref{algorithm:HashNet} will decrease the loss of HashNet~(4) in each stage and in each iteration until converging to HashNet with sign activation function that generates \emph{exactly} binary codes.

Let $L_{ij} = {{w_{ij}}\left( {\log \left( {1 + \exp \left( {\alpha \left\langle {{{\bm{h}}_i},{{\bm{h}}_j}} \right\rangle } \right)} \right) - \alpha {s_{ij}}\left\langle {{{\bm{h}}_i},{{\bm{h}}_j}} \right\rangle } \right)}$ and $L = \sum\nolimits_{{s_{ij}} \in \mathcal{S}} L_{ij}$, where ${\bm h}_i \in \{-1,+1\}^K$ are \emph{binary} hash codes. 
Note that when optimizing HashNet by continuation in Algorithm~\ref{algorithm:HashNet}, network activation in each stage $t$ is $g = \tanh(\beta_t z)$, which is \emph{continuous} in nature and will only become \emph{binary} when convergence ${\beta _t} \to \infty $. Denote by $J_{ij} = {{w_{ij}}\left( {\log \left( {1 + \exp \left( {\alpha \left\langle {{{\bm{g}}_i},{{\bm{g}}_j}} \right\rangle } \right)} \right) - \alpha {s_{ij}}\left\langle {{{\bm{g}}_i},{{\bm{g}}_j}} \right\rangle } \right)}$ and $J = \sum\nolimits_{{s_{ij}} \in \mathcal{S}} J_{ij}$ the true loss we optimize in Algorithm~1, where ${\bm g}_i \in \mathbb{R}^K$ and note that ${\bm h}_i = \operatorname{sgn}({\bm g}_i)$. We will show that HashNet loss $L({\bm h})$ descends when minimizing $J({\bm g})$.

\begin{theorem}\label{the:stage}
The HashNet loss $L$ will not change across stages $t$ and $t+1$ with bandwidths switched from $\beta_t$ to $\beta_{t+1}$.
\end{theorem}
\begin{proof}
When the algorithm switches from stages $t$ to $t+1$ with bandwidths changed from $\beta_t$ to $\beta_{t+1}$, only the network activation is changed from $\operatorname{tanh}(\beta_t z)$  to $\operatorname{tanh}(\beta_{t+1} z)$ but its sign $h = \operatorname{sgn}(\operatorname{tanh}(\beta_t z)) =  \operatorname{sgn}(\operatorname{tanh}(\beta_{t+1} z))$, i.e. the hash code, remains the same. Thus $L$ is unchanged.
\end{proof}

For each pair of binary codes $\bm{h}_i$, $\bm{h}_j$ and their continuous counterparts $\bm{g}_i$, $\bm{g}_j$, the derivative of $J$ w.r.t. each bit $k$ is
\begin{equation}\label{eqn:gradJ}
\frac{{\partial J}}{{\partial {g_{ik}}}} = w_{ij} \alpha \left( {\frac{1}{{1 + \exp \left( {-\alpha  \left\langle {{\bm{g}}_i, {\bm{g}}_j} \right\rangle } \right)}} - {s_{ij}}} \right){g_{jk}},
\end{equation}
where $k=1,\ldots,K$. The derivative of $J$ w.r.t. $\bm{g}_{j}$ can be defined similarly. Updating ${\bm g}_i$ by SGD, the updated ${{{\bm g}'_i}}$ is
\begin{equation}
\begin{aligned}
{g}'_{ik} & = {{g}}_{ik} - \eta \frac{{\partial J}}{{\partial {{{g}}_{ik}}}} \\
  & = {{g}}_{ik} - \eta w_{ij} \alpha \left( {\frac{1}{{1 + \exp \left( {-\alpha  \left\langle {{\bm{g}}_i, {\bm{g}}_j} \right\rangle } \right)}} - {s_{ij}}} \right){g_{jk}},
\end{aligned}
\end{equation}
where $\eta$ is the learning rate and $\bm{g}'_j$ is computed similarly. 

\begin{lemma}\label{lemma:con}
Denote by ${\bm h}_i = \operatorname{sgn}({\bm g}_i)$, ${\bm h}'_i = \operatorname{sgn}({\bm g}'_i)$, then
\begin{equation}
  \begin{cases}
\left\langle {\bm{h}'_i, \bm{h}'_j} \right\rangle \geqslant \left\langle {{\bm{h}}_i, {\bm{h}}_j} \right\rangle, \quad {s_{ij}} = 1,\\ 
\left\langle {\bm{h}'_i, \bm{h}'_j} \right\rangle \leqslant \left\langle {{\bm{h}}_i, {\bm{h}}_j} \right\rangle, \quad {s_{ij}} = 0.\\
  \end{cases}
  \end{equation}
\end{lemma}

\begin{proof}
Since $\left\langle {{\bm{h}}_i, {\bm{h}}_j} \right\rangle  = \sum\nolimits_{k = 1}^K {{h}_{ik} {h}_{jk}}$, Lemma~\ref{lemma:con} can be proved by verifying that ${h'_{ik}}{h'_{jk}} \geqslant {h_{ik}}{h_{jk}}$ if $s_{ij}=1$ and ${h'_{ik}}{h'_{jk}} \leqslant {h_{ik}}{h_{jk}}$ if $s_{ij}=0$, $\forall k = 1,2,\ldots,K$.

\begin{remark} ${s_{ij}} = 0$.

(1) If $g_{ik} < 0$, $g_{jk} > 0$, then $\frac{{\partial J}}{{\partial {g_{ik}}}} > 0$, $\frac{{\partial J}}{{\partial {g_{jk}}}} < 0$. Thus, ${{{h}}'_{ik}} \leqslant {{h}_{ik}}=-1$, ${{{h}}'_{jk}} \geqslant {{h}_{jk}}=1$. And we have ${{{h}}'_{ik}} {{{h}}'_{jk}} = -1 = {{h}_{ik}} {{h}_{jk}}$.

(2) If $g_{ik} > 0$, $g_{jk} < 0$, then $\frac{{\partial J}}{{\partial {g_{ik}}}} < 0$, $\frac{{\partial J}}{{\partial {g_{jk}}}} > 0$. Thus, ${{{h}}'_{ik}} \geqslant {{h}_{ik}}=1$, $ {{{h}}'_{jk}} \leqslant {{h}_{jk}}=-1$. And we have ${{{h}}'_{ik}} {{{h}}'_{jk}} = -1 = {{h}_{ik}} {{h}_{jk}}$.

(3) If $g_{ik} < 0$, $g_{jk} < 0$, then $\frac{{\partial J}}{{\partial {g_{ik}}}} < 0$, $\frac{{\partial J}}{{\partial {g_{jk}}}} < 0$. Thus ${{{h}}'_{ik}} \geqslant {{h}_{ik}}=-1$, $ {{{h}}'_{jk}} \geqslant {{h}_{jk}}=-1$. So ${{{h}}'_{ik}}$ and ${{{h}}'_{jk}}$ may be either $+1$ or $-1$ and we have ${{{h}}'_{ik}} {{{h}}'_{jk}} \leqslant 1 = {{h}_{ik}} {{h}_{jk}}$.

(4) If $g_{ik} > 0$, $g_{jk} > 0$, then $\frac{{\partial J}}{{\partial {g_{ik}}}} > 0$, $\frac{{\partial J}}{{\partial {g_{jk}}}} > 0$. Thus ${{{h}}'_{ik}} \leqslant {{h}_{ik}}=1$, $ {{{h}}'_{jk}} \leqslant {{h}_{jk}}=1$. So ${{{h}}'_{ik}}$ and ${{{h}}'_{jk}}$ may be either $+1$ or $-1$ and we have ${{{h}}'_{ik}} {{{h}}'_{jk}} \leqslant 1 = {{h}_{ik}} {{h}_{jk}}$.
\end{remark}

\begin{remark} ${s_{ij}} = 1$. It can be proved similarly as Case 1.
\end{remark}
\vspace{-10pt}
\end{proof}

\begin{theorem}\label{the:converge}
Loss $L$ decreases when optimizing loss $J({\bm g})$ by the stochastic gradient descent (SGD) within each stage.
\end{theorem}

\begin{proof}
The gradient of loss $L$ w.r.t. \emph{hash} codes $\left\langle {{\bm{h}}_i, {\bm{h}}_j} \right\rangle$ is
\begin{equation}\label{eqn:dinnerproduct}
\frac{{\partial L}}{{\partial  \left\langle {{\bm{h}}_i, {\bm{h}}_j} \right\rangle }} = w_{ij} \alpha \left( {\frac{1}{{1 + \exp \left( {-\alpha  \left\langle {{\bm{h}}_i, {\bm{h}}_j} \right\rangle } \right)}} - {s_{ij}}} \right). 
\end{equation}
We observe that
\begin{equation}\label{eqn:dinnerproduct_case}
  \begin{cases}
    \frac{{\partial L}}{{\partial  \left\langle {{\bm{h}}_i, {\bm{h}}_j} \right\rangle }} \leqslant 0,\quad {s_{ij}} = 1, \\
    \frac{{\partial L}}{{\partial  \left\langle {{\bm{h}}_i, {\bm{h}}_j} \right\rangle }} \geqslant 0,\quad {s_{ij}} = 0. \\ 
  \end{cases}
\end{equation}
By substituting Lemma~\ref{lemma:con}: if ${s_{ij}} = 1$, then $\left\langle {\bm{h}'_i, \bm{h}'_j} \right\rangle \geqslant \left\langle {{\bm{h}}_i, {\bm{h}}_j} \right\rangle$, and thus $L(\bm{h}'_i, \bm{h}'_j) \leqslant L(\bm{h}_i, \bm{h}_j)$; if ${s_{ij}} = 0$, then $\left\langle {\bm{h}'_i, \bm{h}'_j} \right\rangle \leqslant \left\langle {{\bm{h}}_i, {\bm{h}}_j} \right\rangle$, and thus $L(\bm{h}'_i, \bm{h}'_j) \leqslant L(\bm{h}_i, \bm{h}_j)$.
\end{proof}

\bibliographystyle{ieee}
\bibliography{ZCao}

\end{document}